%% file: ppopp2021-sigplan.tex
\documentclass[sigplan]{acmart}

\AtBeginDocument{%
  \providecommand\BibTeX{{%
    \normalfont B\kern-0.5em{\scshape i\kern-0.25em b}\kern-0.8em\TeX}}}

\setcopyright{acmcopyright}
\copyrightyear{2021}
\acmYear{2021}
\acmDOI{10.1145/1122445.nnnnnnn}

\acmConference[Technical Report]{}{2021}{}

\usepackage{booktabs}   
\usepackage{subcaption} 
\usepackage{natbib}
\usepackage{amsmath}
\usepackage{mathrsfs}
\usepackage{algorithm}
\usepackage{algorithmic}
\usepackage{multirow}
\usepackage{color}
\usepackage{tabularx}

\begin{document}
\title[I/O Lower Bounds for Auto-tuning of Convolutions in CNNs]{I/O Lower Bounds for Auto-tuning of Convolutions in CNNs}
\subtitle {\Large Technical Report}

\author{Xiaoyang Zhang, Junmin Xiao$^{*}$, and Guangming Tan}
\affiliation{%
	\institution{%
		State Key Laboratory of Computer Architecture,
		Institute of Computing Technology, Chinese Academy of Sciences}
	\country{University of Chinese Academy of Science}\\
	zhangxiaoyang@ncic.ac.cn \quad xiaojunmin@ict.ac.cn \quad tgm@ict.ac.cn
}

\renewcommand{\shortauthors}{Xiaoyang Zhang, Junmin Xiao, and Guangming Tan}

\begin{abstract}
Convolution is the most time-consuming part in the computation of convolutional neural networks (CNNs), which have achieved great successes in numerous practical applications. Due to the complex data dependency and the increase in the amount of model samples, the convolution suffers from high overhead on data movement (i.e., memory access). This work provides comprehensive analysis and methodologies to minimize the communication for the convolution in CNNs. With an in-depth analysis of the recent I/O complexity theory under the red-blue game model, we develop a general I/O lower bound theory for a composite algorithm which consists of several different sub-computations. Based on the proposed theory, we establish the data movement lower bound results for two main convolution algorithms in CNNs, namely the direct convolution and Winograd algorithm, which represents the direct and indirect implementations of a convolution respectively. Next, derived from I/O lower bound results, we design the near I/O-optimal dataflow strategies for the two main convolution algorithms by fully exploiting the data reuse. Furthermore, in order to push the envelope of performance of the near I/O-optimal dataflow strategies further, an aggressive design of auto-tuning based on I/O lower bounds, is proposed to search an optimal parameter configuration for the direct convolution and Winograd algorithm on GPU, such as the number of threads and the size of shared memory used in each thread block. Finally, experiment evaluation results on the direct convolution and Winograd algorithm show that our dataflow strategies with the auto-tuning approach can achieve about $3.32 \times$ performance speedup on average over cuDNN. In addition, compared with TVM, which represents the state-of-the-art technique for auto-tuning, not only our auto-tuning method based on I/O lower bounds can find the optimal parameter configuration faster, but also our solution has higher performance than the optimal solution provided by TVM.
\end{abstract}

\keywords{I/O lower bounds, red-blue pebble game, dataflow design, auto-tuning, convolutional neural network.}


\maketitle

\input{ppopp2021-conference}
\bibliographystyle{ACM-Reference-Format}
\bibliography{ppopp2021}


\end{document}

%% file: ppopp2021-conference.tex
\section{Introduction}

Convolutional neural networks (CNNs) are commonly applied to numerous computer vision and machine learning applications, which have achieved great successes because the complex layer structures could produce high-quality results based on a large number of data. Specifically, the convolution layer is an important structure in many state-of-the-art modern CNN models, such as MobileNet \cite{howard2017mobilenets}, ResNet \cite{szegedy2016inception}, ShuffleNet \cite{zhang2018shufflenet}, SqueezeNet \cite{iandola2016squeezenet}, VggNet\cite{simonyan2015deep} and so on. The wide adoption of convolution and its huge cost have led to a high demand to optimize convolution operations for high performance. From the hardware perspective, GPUs have been demonstrated to be able to provide tremendous computation power for accelerating convolution operations \cite{yan2020ppopp}. Furthermore, many specific accelerators for convolutions in CNNs are designed based on field-programmable gate arrays (FPGA) and application-specific integrated circuits (ASIC). From the software perspective, a variety of optimization techniques have been developed from algorithm level \cite{yucheng2017corr} to compilation level \cite{zhao2020mirco}. Many optimization efforts have also been incorporated to the widely used software libraries, such as NVIDIA cuDNN \cite{sharan2014corr} and AMD MIOpen \cite{jehandad2019corr}.

For convolution operations in CNNs, multiple convolution algorithms have been developed and classified into two categories: direct and indirect approaches. Typical direct and indirect representatives are the direct convolution and Winograd convolution algorithms respectively, each of which involves a large amount of memory accesses due to the complex computational workflow and massive data in convolution operations. For example, all inputs and weights are typically stored in the off-chip memory of CNN accelerators, such as global memory in GPUs. During computation, partial inputs and weights are loaded from the off-chip memory into the on-chip buffer to produce portions of outputs. Meanwhile, each processor could use its own registers to read some inputs and weights which have been in the on-chip buffer. Consequently, the frequent data movement in the memory hierarchy commonly dominates the energy consumption in convolution operations \cite{chen2020hpca}. Therefore, optimizing the data transmission in convolutions is the key for improving the performance of convolutions.

To minimize data movement, the most works focus on how to reduce the model size, such as quantifying weights \cite{zhou2017incremental}. On the other hand, another effective way for reducing communication is to increase data reuse based on the dataflow design. In recent years, a variety of dataflow approaches have been proposed \cite{chen2016eyeriss,shah2018runtime,jo2018energy}, most of which mainly focus on the computation efficiency. However, the data movement of convolutions has not been taken a full account. This work would try to consider the communication-optimal strategies for different convolution algorithms based on the I/O lower bound analysis.

Since I/O lower bound analysis is important for evaluating the optimality of a proposed algorithmic solution, it is widely concerned to establish appropriate lower bounds of the data movement of application codes \cite{xiao2019ccfhpc,xiao2018icpp}. Under the red-blue pebble game model \cite{jia1981complexity} for data transmission in memory hierarchy, past work on I/O lower bounds has found bounds for specific algorithms, such as matrix-matrix multiplication and FFT. As the recent methodology mainly focuses on the workflow’s specific properties which do not translate across different computational patterns, the recent lower bound theory seems hard to be applied to arbitrary computations such as convolutions, in which different sub-computations involve different computational patterns. How to establish a systematic I/O lower bound theory for convolutions based on the red-blue pebble game model is a big challenge\cite{zhang2020spaa}. Even if the lower bounds could be obtained, the theoretical minimum of I/O complexity is not easy to directly yield an efficient dataflow strategy. There is a very large space to explore. How to determine the 
dataflow with the help of I/O lower bound is another challenge.

To solve the above challenges, this work considers to quantify the contribution of each sub-computation to the total computation, and then generalizes the recent I/O lower bound theory to establish I/O lower bound results for convolutions under the red-blue pebble game model. Next, through a deeper investigation of the highest order term in the lower bound results, we determine which data reuse should be fully exploited, and propose I/O-optimal dataflow strategy for maximizing such data reuse to minimize the memory access in convolutions. Furthermore, by comparing the lower bound result with I/O cost of our dataflow strategy, the optimality condition for implementation of convolutions is deduced. Based on the optimality condition, a fine-grained auto-tuning optimization is designed to effective find the optimal implementation with high performance.

In this work, we make the following key contributions.

\begin{itemize}
	\item Develop a general I/O lower bound theory for any arbitrary composite algorithm which involves different sub-computations and different computational patterns, under the red-blue pebble game model.
	\vspace{0.25cm}
	\item Establish I/O lower bound results for two typical representatives of direct and indirect convolution algorithms, which are the direct convolution and Winograd convolution algorithms.
	\item Design near I/O-optimal dataflow strategies respectively for the direct convolution and Winograd convolution algorithms.
	\item Propose an auto-tuning engine to achieve excellent implementations of our dataflow strategies.
\end{itemize}

\section{Background}

\subsection{Red-blue Pebble Game}
\label{subsection: Red-blue Pebble Game}

The red-blue pebble game is a two-level memory access model which is proposed by Hong \& Kung. This model is suitable for small-fast and large-slow memory structures and our theoretical analysis of lower bound is based on it. Red and blue pebbles represent fast storage and slow storage, respectively. The fast storage is limited, thus the number of red pebbles is small. The slow storage is unlimited, and there is no limit to the number of blue pebbles. The game is played on a directed acyclic graph (DAG), and DAG describes the operation of the algorithm. Furthermore, the rules of a red-blue pebble game are as follows:
\begin{itemize}
	\item (Load) A red pebble may be placed on any vertex that has a blue pebble.
	\item (Store) A blue pebble may be placed on any vertex that has a red pebble.
	\item (Compute) If all the immediate predecessors of a vertex have red pebbles, a red pebble may be placed on such vertex.
	\item (Free) A pebble no matter red or blue may be removed from any vertex.
\end{itemize}

Let $G(V,E)$ be a DAG, which describes an algorithm. $V$ is the vertex set representing operations of algorithm, and $E$ is the edge set representing the dependency of two operations. A partition on $G$ is called an S-partition, if the following four properties hold.

\begin{itemize}
	\item Property 1: $V$ is partitioned into $h$ subsets $V_1,V_2, \cdots, V_h$ such that $V_i$'s are disjoint but their union is $V$.
	\item Property 2: There is a dominator set $D_i$ for each $V_i$ that contains at most $S$ vertices. A dominator set $D_i$ for $V_i$ is a set of nodes in $V$ such that any path from an input of $G$ to a node in $V_i$ contains some nodes in $D_i$.
	\item Property 3: There is a minimum set $M_i$ for each $V_i$ that contains at most $S$ vertices. The minimum set of $V_i$ is defined to be the set of vertices in $V_i$ that do not have any successor vertex belonging to $V_i$.
	\item Property 4: No cyclic dependence is among $V_1,\cdots,V_h$.
\end{itemize}

Let $P(S)$ be the minimum number of subsets that any S-partition of a DAG must have. The following theorem describes the communication lower bound based on the S-partition model (for the proof, refer to \cite{jia1981complexity}).

\begin{theorem}
	\label{theorem: original lower bound}
	Any complete calculation of a red-blue pebble game on DAG $G=(V,E)$ with at most $S$ red pebbles needs the minimum I/O time $Q$ such that
	\begin{equation}
	\label{equation: original lower bound}
	Q \geq S \cdot (P(2S)-1).
	\end{equation}
\end{theorem}

\subsection{Direct Convolution}

\begin{figure}[ht]	
	\centering
	\includegraphics[scale=0.4]{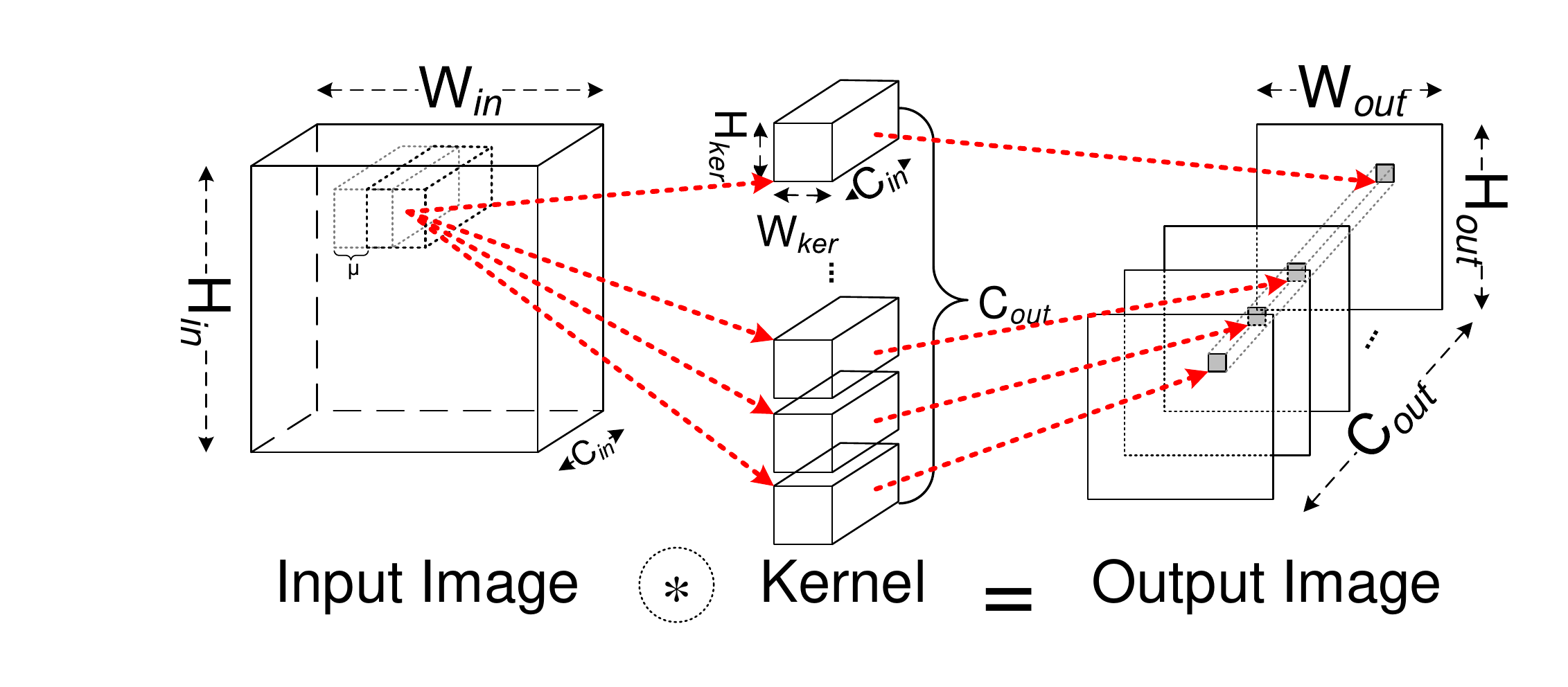}
	\caption{Direct Convolution.}
	\label{fig:DirectConvolution}
\end{figure}


Figure \ref{fig:DirectConvolution} illustrates a direct convolution. We have an input image of size $W_{in} \times H_{in} \times C_{in}$ and $C_{out}$ kernels of weights, producing a $W_{out} \times H_{out} \times C_{out}$ output image. For the convolution, the channels of input image $C_{in}$ is the number of channels in each kernel, and the channels of output image $C_{out}$ is equal to the number of kernels, and each channel of output image is a $H_{out} \times W_{out}$ matrix. The kernel is a $W_{ker} \times H_{ker} \times C_{in}$ tensor. Each output is computed by an inner product between a kernel tensor and a sliding input tensor with the size of $W_{ker} \times H_{ker} \times C_{in}$ from an input image by using a sliding window. The stride size $\mu$ is the position difference between two adjacent sliding windows.

\subsection{Winograd Algorithm}
\label{section: Winograd Algorithm}

Winograd algorithm for convolution is shown in Figure \ref{fig:WinogradAlgorithm}. This algorithm changes the characteristics of time-domain convolution calculations, and reduces the number of multiplication operations between input images and kernels through mathematical transformation. In order to perform the mathematical transformation, several parameter matrices are introduced. Matrix $A$, $B$ and $L$ are three transformation matrices for output images, input images and kernels respectively. Furthermore, as Winograd algorithm requires $W_{ker}=H_{ker}$, we denote $r$ as $W_{ker}$ or $H_{ker}$ briefly. Winograd algorithm can calculate multiple output results at once. Here, we denote $F(e \times e,r \times r)$ as a calculation process to deduce $e^2$ outputs in winograd algorithm. Theoretically, the value of $e$ is arbitrary, but in practice $e$ usually is chosen as $2$, $3$ or $4$. To compute every $e^2$ outputs at a fixed channel of an output image, $F(e \times e, r \times r)$ requires a sliding input tensor with the size of $(e+r-1) \times (e+r-1) \times C_{in}$ from input images using a sliding window and a kernel tensor with the size of $e \times e \times C_{in}$. Then the input tensor and kernel are transformed by $B$ and $L$ into $P$ and $J$ which have the same size of $(e+r-1) \times (e+r-1) \times C_{in}$. Next, the corresponding element product of $P$ and $J$ results in a new  $(e+r-1) \times (e+r-1) \times C_{in}$ tensor $\varLambda$, and the summation of elements in $\varLambda$ along channel direction generates a $(e+r-1) \times (e+r-1)$ matrix $\varPi$. Finally, $\varPi$ is transformed by $A$ into a $e\times e$ matrix which are $e^2$ outputs. 

\begin{figure}[ht]	
	\centering
	\includegraphics[scale=0.35]{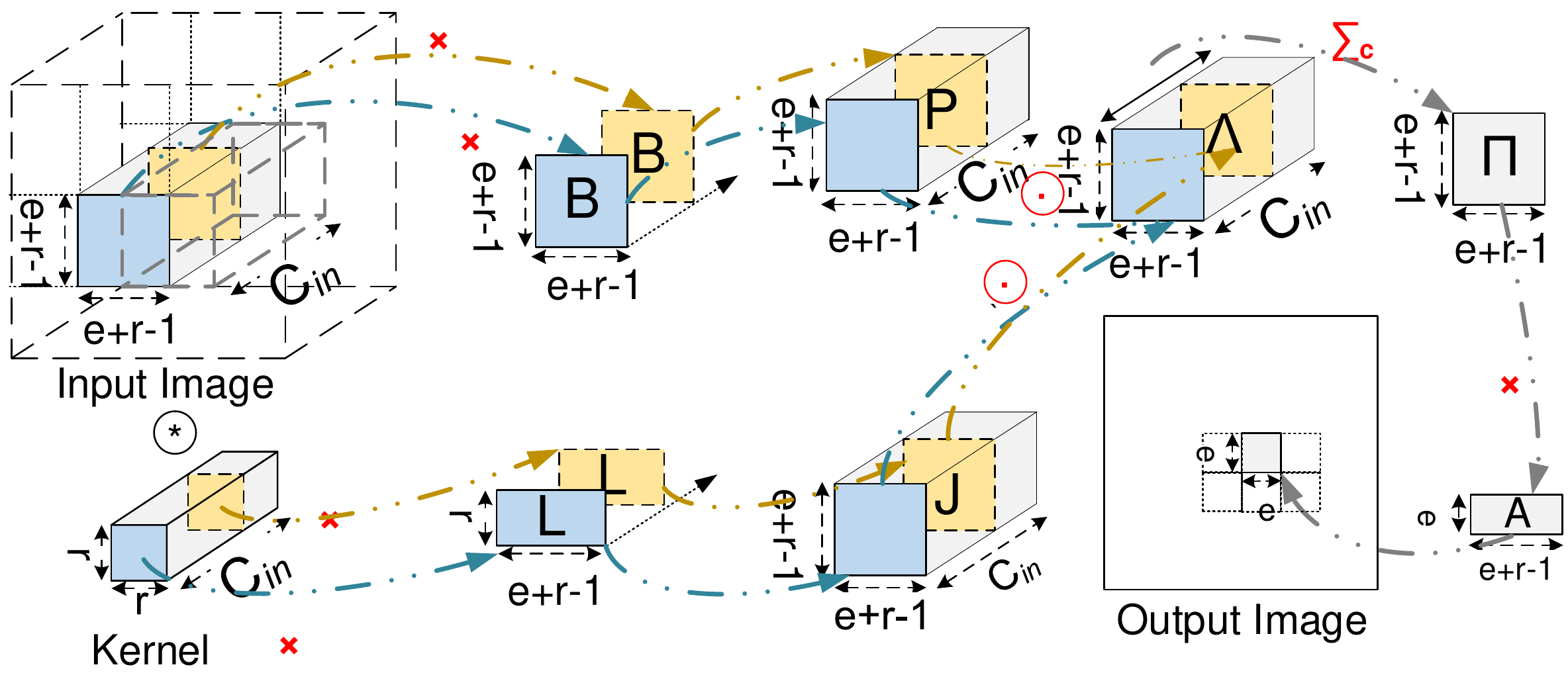}
	\caption{Winograd Algorithm.}
	\label{fig:WinogradAlgorithm}
\end{figure}

\section{Motivation}

In this section, we elaborate specific challenges that need to be addressed in order to build I/O lower bound theory and design I/O optimal dataflow implementations, and present our basic idea to address these challenges.

\subsection{Challenges for I/O Lower Bounds Building} 

\begin{figure}[htbp]
	\centering
	\includegraphics[scale=0.15]{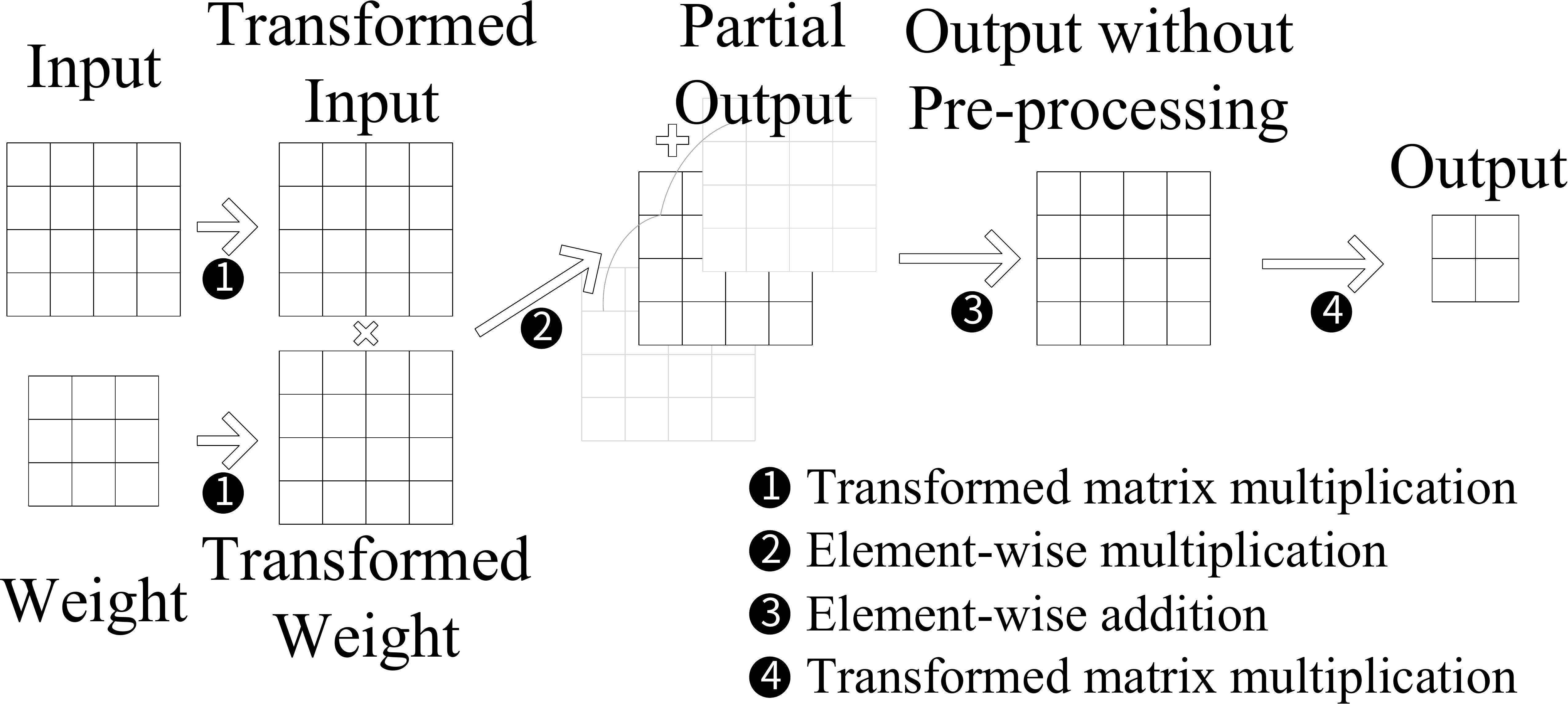}
	\caption{Different Patterns for Winograd Algorithm.}
	\label{fig:MultistepChallenge}
\end{figure}

Since the lower bound analysis is important for evaluating the optimality of a proposed algorithmic solution, and often yields new insights in algorithm optimization for achieving better performance, it is widely concerned to establish appropriate I/O lower bounds of the data movement of a numerical method on a hierarchical memory system. In real application, most numerical algorithms, such as convolutions, are typically constructed from a number of sub-computations. For instance, Winograd algorithm has 4 sub-computations, which involve $4$ different patterns (Figure \ref{fig:MultistepChallenge}): (1) transformed matrix multiplication, (2) element-wise multiplication, (3) element-wise addition, (4) transformed matrix multiplication. Although the red-blue pebble game model has been proposed for many years, it is still difficult to use this model to establish I/O lower bounds of composite algorithms which involve several different kinds of computational patterns \cite{elango2014characterizing}.
It is not even possible to deduce a suitable I/O lower bound of the DAG only focusing on each sub-computation of the composite codes, due to the following two main reasons. 
Firstly, at the beginning of the red-blue pebble game, all DAG vertices without predecessors have blue pebbles, and all vertices without successors would get blue pebbles at the end of the game. Based on this assumption, the calculation for each sub-DAG will require at least one load operation for each input and one store operation for each output. However, when the red-blue pebble game is played on the full DAG, the data could pass from a previous sub-computation to a later one directly through fast memory. Secondly, when a composite computation is assigned into several sub-computations, the total DAG is partitioned into several relevant sub-graphs. Under a common constraint that previous sub-computation must be totally finished before the later sub-computation starts, the partition way of the total DAG usually impacts the data movement complexity due to the limited size of fast memory. The two reasons above describe the essential difficulties to develop the general I/O lower bound theory for any composite algorithm. To get around these difficulties, the red-blue-white pebble game model has been proposed recently to analyze composite algorithms, which uses some restrictions on models \cite{elango2014characterizing}.

\subsection{Challenges for Optimal Implementations}

If I/O lower bound can be obtained, it often provides some insights for the algorithm design. For instance, I/O lower bound can tell us which data should be reused in the on-chip memory prior to the others (see Section \ref{section:I/O Optimal Dataflow}). When we know which data has higher reuse priority, another challenge is designing the optimal implementation to maximize such data reuse. For the implementation design of convolutions, the combinatorial choices of memory access, threading pattern, specific input shape and layout create a huge configuration space, such as loop tiling, ordering, unrolling, and so on. For 4 sub-computations in Winograd algorithm, the size configuration space is usually larger than $10^6$. This fact indicates that it is hard to manually design an efficient implementation for a convolution. Although NVIDIA proposes excellent implementations for different convolution algorithms in cuDNN library \cite{choi2010model}, these implementations mainly focus on general optimization on GPUs. Directly using the convolution API in cuDNN sometimes can not satisfy the real-time demand of inference applications. Recently, auto-tuning methods have been proposed for the fine-grained optimization of convolutions. The common way is to adopt a predefined cost model to guide the search, but building an accurate cost model is difficult due to the increasing complexity of modern hardware. As the state-of-the-art framework for auto-tuning convolutions, TVM proposes a new auto-tuning method based on ML-model \cite{chen2018tvm}. However, it still needs a large search cost due to the huge search space.

\subsection{Basic Idea}

In this work, we explore the red-blue pebble game. The analysis on each subcomputation could not accurately estimate the data movement complexity, which is because the contribution of each sub-computation to total computation is ignored.
Through the quantification of such contribution, all sub-computations can be viewed as a whole, which provides an opportunity to build I/O lower bound of composite algorithms.

Besides, to addresses the challenges for algorithm optimization, this work combines both the coarse-grained design and fine-grained optimization for convolutions. 
Based on the lower bound analysis, we propose a coarse-grained dataflow design by fully exploiting the data reuse and minimizing the off-chip memory access. With comparing the I/O volume of the dataflow with the lower bound, we discover the optimality condition for I/O optimal design. Using this optimality condition, our fine-grained optimization is to reduce the size of search space and proposing an effective parallel searching method to find the optimal implementation, which leads to an auto-tuning engine.

\section{Lower Bound Theory}

\subsection{Red-Blue Pebble Game Re-exploration}
\label{section: Red-Blue Pebble Game Re-exploration}

\subsubsection{Basic Idea}
\label{subsubsection: Basic Idea}

In order to build I/O lower bounds of
convolutions, we revisit the red-blue pebble game in fine-grained. First of all, it is not easy for a composite code to deduce the value of $P(S)$ indeed, while we could try to estimate a valid lower bound of $P(S)$. Denote $\mathbb{P}_{S}$ as the set containing all possible options of S-partitions for DAG $G(V,E)$, and each element in $\mathbb{P}_{S}$ represents one S-partition of $G(V,E)$. Let
\begin{equation}
\label{equation: HS}
H(S) = \min_{\{V_{1},\cdots, V_{h}\} \in \mathbb{P}_{S} }  \dfrac{|V|}{\max_{1\leq i \leq h} |V_{i}|}.
\end{equation}
According the definition of $P(S)$ in Section \ref{subsection: Red-blue Pebble Game}, we have $P(S) \geq H(S)$. This fact, together with Equations (\ref{equation: original lower bound}) and (\ref{equation: HS}), implies that the I/O time $Q$ satisfies 
\begin{equation}
\label{equation: QQ}
Q \geq S \cdot (P(2S)-1) \geq S \cdot (H(2S)-1).
\end{equation}
Hence, we only need to estimate $H(2S)$ instead of $P(2S)$. Secondly, from Equation (\ref{equation: HS}), $H(S)$ depends on the value of $\max_{1\leq i \leq h} |V_{i}|$, which means that the fine-gained analysis on $V_{i}$ is the key. Thirdly, if we can find out the relationship between $V_i$ and all sub-computations of $G(V,E)$, it would become possible to estimate the number of vertices in $V_i$. Before deducing the upper bound of $|V_i|$, we formalize the notation of multi-step partition of a DAG.
\begin{definition}
	\label{definition: multi-step partition of DAG}
	Assume that a DAG $G(V,E)$ is decomposed into $n$ sub-DAGs $G_1(U_{1}, E_{1}), G_2(U_{2}, E_{2}), \cdots, G_n(U_{n},E_{n})$ where $ G_j(U_{j}, E_{j})$ is corresponding to a sub-computation. $\{G_1(U_{1}, E_{1})$, $\cdots, G_n(U_{n},E_{n}) \}$ is called as \textit{a multi-step partition} of $G(V,E)$, if and only if any input vertex of $G_j(U_{j},E_{j})$ must be an output vertex of $G_{j-1}(U_{j-1},E_{j-1})$, and the internal vertex sets of all $U_j$'s are disjoint from each other.
\end{definition}
It is clear that any sequence of sub-computations can be represented as a multi-step partition of the DAG for the total computation. Assume that $\{G_1(U_{1}, E_{1}), \cdots, G_n(U_{n},E_{n}) \}$ is a multi-step partition of $G(V,E)$. If we are able to estimate all the upper bounds of $|V_{i} \cap U_{j}|$ ($j=1, 2, \cdots, n$) by using Property 2 and Property 3 in the definition of S-partition, it is possible to obtain the maximum of $|V_{i}|$. 

In the following, we study the feasibility on the derivation of the upper bound of $|V_{i}|$ based on recursive analysis. For some $j$-th sub-computation, assume that the upper bound of $|V_{i} \cap U_{j}|$ has been obtained successfully. The next problem is how to estimate $|V_{i} \cap U_{j+1}|$. Since $|V_{i} \cap U_{j}|$ seems not to be associated with the upper bound of $|V_{i} \cap U_{j+1}|$, we have to focus on how the output set of $U_j$ affects the $(j+1)$-th sub-computation. Denote $\widetilde{O}_j$ as the output set of $U_j$, and $D_i$ as a dominator set of $V_i$. Further, we apply a new concept of \textit{vertex generation} to determine the vertecies in $\widetilde{O}_j$ which are associated with $D_i$ and $V_{i}$.

\begin{definition}
	In a DAG $G(V,E)$, \textit{a vertex set $U$ can generate another vertex set $U'$}, if and only if every path from an input of $V$ to a vertex in $U'$ contains some vertex in $U$. Furthermore, $\varTheta(U)$ represents a set containing all vertices which can be generated by $U$.
\end{definition}

It is obvious that the dominator set $D_i$ can generate $V_i$. Furthermore, $\varTheta(D_i) \cap V_i \cap  \widetilde{O}_{j}$ determines the vertecies in $\widetilde{O}_j$ which are associated with $D_i$ and $V_{i}$. $|\varTheta(D_i) \cap V_i \cap  \widetilde{O}_{j}|$ could be used to deduce the upper bound of $|V_i \cap U_{j+1}|$, because that all inputs of $V_i \cap U_{j+1}$ are included in $\varTheta(D_i) \cap V_i \cap \widetilde{O}_{j}$. In conclusion, if we can dedue the upper bounds of $|V_i \cap U_{j}|$ and $|\varTheta(D_i) \cap V_i \cap \widetilde{O}_{j}|$, it is easy to obtain $|V_i \cap U_{j+1}|$ and $|\varTheta(D_i) \cap V_i \cap \widetilde{O}_{j+1}|$ by using $\varTheta(D_i) \cap V_i \cap \widetilde{O}_{j}$ as the inputs for $V_i \cap U_{j+1}$. Based on recursive analysis, all upper bounds of $|V_i \cap U_{j+1}|$ ($j=1, 2, \cdots, n$) can be established, which would lead to the upper bound of $|V_i|$.

After the feasibility analysis above on the derivation of the upper bound of $|V_{i}|$, \textit{we use a simple example to show the intuition of how to obtain the the upper bound of $|V_i|$}. Assume DAG $G(U,E)$ of a composite algorithm has two sub-computations ($G(U,E)= G_1(U_1,E_1)+G_2(U_2,E_2)$). Denote $k^{i}$ as the number of vertices in the dominator $D_i$ of $V_{i}$. According to the definition of S-partition, we have $|D_i| = k^{i} \leq S$, and divide $k^{i}$ into $k^{i} = k^{i}_{1}+k^{i}_{2}$ where $k^{i}_{1}$ is the number of input vertices for $V_{i} \cap U_1$ and $k^{i}_{2}$ is the number of a part of input vertices for $V_{i} \cap U_2$. For any integer $k$, find two functions $\varphi_j(k)$ and $\psi_j(k)$, where $\varphi_j(k)$ represents the maximum of vertices in $U_j$ generated by using $k$ input vertices, and $\psi_j(k)$ represents the maximum of vertices in $\widetilde{O}_{j}$ generated by using $k$ input vertices. Hence, $|\varTheta(D_i) \cap V_{i} \cap U_1|$ is not larger than $\varphi_1(k^{i}_{1})$, and at most $\psi_1(k^{i}_{1})$ vertices are generated as the inputs for $V_{i} \cap U_2$. Hence, there are at most $k^{i}_2+\psi_1(k^{i}_{1})$ input vertices for $V_{i} \cap U_2$. Further, $|\varTheta(D_i) \cap V_i \cap U_2| \leq \varphi_2(k^{i}_2+\psi_1(k^{i}_{1}))$ is valid. Hence, we have 
\begin{eqnarray}
|V_i| & \leq &  |D_i| + |\varTheta(D_i) \cap V_i \cap U_1|+| \varTheta(D_i) \cap V_i \cap U_2 |  \nonumber \\
& \leq & S+\varphi_1(k^{i}_{1})+\varphi_2(k^{i}_2+\psi_1(k^{i}_{1}))
\nonumber \\
& \leq & S+ \max_{k_{1} +k_{2} \leq S} \left(\varphi_1(k_{1})+\varphi_2(k_2+\psi_1(k_{1}))\right).
\nonumber
\end{eqnarray}
Letting $T(S)=S+\max_{k_{1} +k_{2} \leq S} (\varphi_1(k_{1})+\varphi_2(k_2+\psi_1(k_{1})))$, we achieve $|V_i| \leq T(S)$.

Acorrding to the discussion above, we deduce the general I/O lower bound result of any composite code by three steps. Firstly, find two functions which can determine the numbers of vertices generated by $D_i$ in $V_i \cap U_{j}$ and $V_i \cap \widetilde{O}_{j}$ respectively (Section \ref{subsubsection: Vertex Generation Functions}). Secondly, deduce the upper bound of $|V_i|$ by using the upper bounds of the two functions (Section \ref{subsubsection: Estimation of Upper Bound of Vi}). Finall, establish general I/O lower bound result by substituting the upper bound of $|V_i|$ into Equations (\ref{equation: HS}) and (\ref{equation: QQ}) (Section \ref{subsubsection: IO lower bound results}).

\subsubsection{Two Maximum Vertex Generation Functions}
\label{subsubsection: Vertex Generation Functions}


For any integer $k$ and a vertex set $U$ with any dominator set $D$ satisfying $|D \cap U_{j}| + | \varTheta(D) \cap \widetilde{O}_{j-1} | \leq k$, define \textit{two vertex generation functions} for the $j$-th sub-computation, as follows
\begin{displaymath}
\widetilde{\varphi}_{j} (U, k) = | \varTheta(D) \cap U \cap U_{j}| ~\hbox{and}~ \widetilde{\psi}_{j} (U, k) = | \varTheta(D) \cap U \cap \widetilde{O}_{j}|.
\end{displaymath}
Here, $\widetilde{\varphi}_{j}$ and $\widetilde{\psi}_{j}$ represent the numbers of vertices generated by $D$ in two sub-graphs $U \cap U_{j}$ and $U \cap \widetilde{O}_{j}$ respectively. Furthemore, for any given $k$ and the $j$-th sub-computation, we define 
\textit{maximum vertex generation functions} as
\begin{equation}
	\label{equation: varphi and psi}
	\varphi_{j} (k) = \max_{U} \widetilde{\varphi}_{j} (U, k) ~\hbox{and}~\psi _{j} (k) = \max_{U} \widetilde{\psi}_{j} (U, k).
\end{equation}
It is clear that $\varphi_{j}$ and $\psi_j$ provide the upper bound estimation on the number of vertices in $U_{j}$ and $\widetilde{O}_{j}$, which can be generated by a vertex $D$ satisfying $|D \cap U_{j}| + | \varTheta(D) \cap \widetilde{O}_{j-1} | \leq k$. With maximum vertex generation functions $\varphi_{j}$ and $\psi_j$, it becomes possible to estimate $|V_i \cap U_j|$ ($j =1,\cdots, n$) one by one.

\subsubsection{Estimation of Upper Bound of $|V_i|$}
\label{subsubsection: Estimation of Upper Bound of Vi}

The analysis in Section \ref{subsubsection: Basic Idea}
inspires us that I/O lower bound establishment is equivalent to finding the upper bound of $|V_{i}|$. Further, two kinds of maximum vertex generation functions $\varphi_{j}$ and $\psi_{j}$ in Section \ref{subsubsection: Vertex Generation Functions}, provide us a powerful tool to respecitvely estimate the numbers of vertices in $V_{i} \cap U_{j}$
and $V_{i} \cap \widetilde{O}_{j}$, which are generated by any dominator $D_i$ of $V_i$. In the following, we try to deduce the upper bound of $|V_i|$.

First of all, we deduce two auxiliary results. Let $\widetilde{O}^{i}_{j}$ be the subset of $\widetilde{O}_{j}$ such that for any $v \in \widetilde{O}^{i}_{j}$, any path from the input set of $V$ to $v$ has at least one vertex which belongs in $\cup^{j}_{k=1} (D_i \cap U_k)$. 

\begin{lemma}
	\label{lemma: dominator set of Oij}
	$\widetilde{O}^{i}_{j} \cup (D_i \cap U_{j+1})$ is a dominator set of $\widetilde{O}^{i}_{j+1}$.
\end{lemma}
\begin{proof}
	For each $v \in \widetilde{O}^{i}_{j+1}$, denote $P$ as any path from the input set of $V$ to the vertex $v$. In order to prove Lemma \ref{lemma: dominator set of Oij}, we need to prove that the path $P$ has a vertex which belongs in $\widetilde{O}^{i}_{j} \cup (D_i \cap U_{j+1})$. 
	
	On one hand, if all vertices in $P$ belong in $U_{j+1}$, by the definition of $\widetilde{O}_{j+1}$, there must exist a vertex $u \in D_i \cap U_{j+1}$ on $P$ due to $v \in \widetilde{O}^{i}_{j+1}$. On the other hand, if there is a vertex $p \notin U_{j+1}$ on the path $P$, $p$ belongs in $U_{k}$ ($1 \leq k \leq j$). Hence, the path $P$ would be joint with $\widetilde{O}_j$ which is the output set of $U_j$. Let $w$ be one vertex in $P \cap \widetilde{O}_j $. If the sub-path of $P$ from $w$ to $v$ has a vertex $u$ in $D_i$, then we have $u \in D_i \cap U_{j+1}$ similar to the discussion above. Otherwise, if the sub-path of $P$ from $w$ to $v$ has no vertex in $D_i$, then it is clear that $w$ must belong in $\widetilde{O}^{i}_{j}$. In fact, when the sub-path of $P$ from $w$ to $v$ has no vertex in $D_i$, any path from the input set of $V$ to $w$ must has a vertex in $D_i$. Furthermore, since $w \in \widetilde{O}_j$, any path from the input set of $V$ to $w$ must has non vertex in $D_i \cap U_{k}$ ($k \geq j+1$). Therefore, on each path from the input set of $V$ to $w$, any vertex in $D_i$ must belong in $\cup^{j}_{k=1} (D_i \cap U_k)$. Hence, we have $w \in \widetilde{O}^{i}_{j}$. In conclusion, Lemma \ref{lemma: dominator set of Oij} is valid. 
\end{proof}

With a similar discussion in the proof above, we can find out a dominator set of $V_i \cap U_{j+1}$.

\begin{lemma}
	\label{lemma: dominator set of ViUj}
	$\widetilde{O}^{i}_{j} \cup (D_i \cap U_{j+1})$ is also a dominator of $V_i \cap U_{j+1}$.
\end{lemma}

By Lemma \ref{lemma: dominator set of Oij} and Lemma \ref{lemma: dominator set of ViUj}, we can obtain an upper bound of $|V_i|$, which is improtant for I/O complexity analysis under the red-blue pebble game model. 

\begin{theorem}
\label{theorem: upper bound of Vi}
Assume that $\{G_1(U_{1}, E_{1}), \cdots, G_n(U_{n},E_{n}) \}$ is a multi-step partition of a DAG $G(V,E)$. For any S-partition $\{V_{1},\cdots, V_{h} \}$ of $G(V,E)$, $|V_{i}|$ has an upper bound
\begin{eqnarray}
\label{equation: upper bound2 of Vi}
T(S) = S + \max_{\sum^{n}_{j=1} k_j \leq S}  ( \varphi_{1} (k_1) + \varphi_{2} (k_2 + \psi_{1}(k_1)) + \cdots \cdots \nonumber \\ 
+ \varphi_{n} (k_n + \psi_{n-1}(k_{n-1} + \psi_{n-2}(k_{n-2} \cdots +  \psi_{1}(k_{1}))))). 
\end{eqnarray}
\end{theorem}
\begin{proof}
For any $V_i$ in the S-partition $\{V_{1},\cdots, V_{h} \}$ of $G(V,E)$ , let $k^i_j= |D_{i} \cap U_{j}|$. In the following, we prove $|V_i| \leq T(S)$ by three steps. First of all, we prove that
\begin{equation}
\label{equation: upper bound of Oij}
|\widetilde{O}^i_j| \leq \psi_{j}(k^i_{j} + \psi_{j-1}(k^i_{j-1} \cdots +  \psi_{1}(k^i_{1})))),
\end{equation}
is valid for any integer $j \in [1, n]$ by using the mathematical induction. When $j =1$, it is obvious that $D_i \cap U_1$ is a dominator set of $\widetilde{O}^i_1$. Since $|D_{i} \cap U_{1}|= k^i_1$, we have $|\widetilde{O}^i_1| \leq \psi_1( k^i_{1} )$. This implies that the inequality (\ref{equation: upper bound of Oij}) is valid for $j =1$. Assume that the inequality (\ref{equation: upper bound of Oij}) is valid for $j \geq 1$. We need to prove that the result is also valid for $j+1$. In fact, by Lemma \ref{lemma: dominator set of Oij}, $\widetilde{O}^{i}_{j} \cup (D_i \cap U_{j+1})$ is a dominator set of $\widetilde{O}^{i}_{j+1}$. Furthermore, using the assumption above, we have
\begin{eqnarray}
\label{equation: upper bound of OikDiUk1}
|\widetilde{O}^i_j \cup (D_i \cap U_{j+1})| & \leq &  |D_i \cap U_{j+1}| + |\widetilde{O}^i_j| \\
& \leq & k^{i}_{j+1} +
\psi_{j}(k^i_{j} + \psi_{j-1}(k^i_{j-1} \cdots +  \psi_{1}(k^i_{1}))). \nonumber
\end{eqnarray}
Furthermore, the definition of $\psi_{j}$ leads to 
\begin{equation}
\label{equation: upper bound of Oik1}
|\widetilde{O}^i_{j+1}| \leq \psi_{j+1}( k^{i}_{j+1} + \psi_{j}(k^i_{j} + \psi_{j-1}(k^i_{j-1} \cdots +  \psi_{1}(k^i_{1})))).
\end{equation}
Thus, the inequality (\ref{equation: upper bound of Oij}) is valid for $j+1$.
	
Next, by the inequality (\ref{equation: upper bound of Oij}), we can further check that the inequality (\ref{equation: upper bound of OikDiUk1}) is always valid. According to Lemma \ref{lemma: dominator set of ViUj}, $\widetilde{O}^{i}_{j} \cup (D_i \cap U_{j+1})$ is also a dominator set of $\varTheta(D_i) \cap V_i \cap U_{j+1}$. Hence, we have 
\begin{equation}
\label{equation: upper bound of ViU1}
|\varTheta(D_i) \cap V_i \cap U_{1}| \leq \varphi_{1} (|D_{i} \cap U_1|)  \leq  \varphi_{1} (k^{i}_{1}),
\end{equation}
and
\begin{eqnarray}
\label{equation: upper bound of ViUj1}
|\varTheta(D_i) \cap V_i \cap U_{j+1}| \leq  \varphi_{j+1} (|\widetilde{O}^{i}_{j} \cup (D_i \cap U_{j+1})|) \nonumber \\
\leq 
\varphi_{j+1} (k^i_{j+1} + \psi_{j}(k^i_{j} + \psi_{j-1}(k^i_{j-1} \cdots +  \psi_{1}(k_{1})))). 
\end{eqnarray}

Finally, since $V = \cup^{n}_{j =1} U_{j}$, $V_i = \cup^{n}_{j =1} (V_i \cap U_{j})$ is valid. As each vertex of $D_i$ can possibly be a vertex in $V_i$, we have
\begin{displaymath}
|V_i| \leq S + | \varTheta(D_i) \cap V_i \cap U_{1} | + | \varTheta(D_i) \cap V_i \cap U_{2}| + \cdots +| \varTheta(D_i) \cap V_i \cap U_{n}|,
\end{displaymath}
which together with (\ref{equation: upper bound of ViU1}) and (\ref{equation: upper bound of ViUj1}), implies that $|V_i| \leq T(S)$ is valid.
\end{proof}

\subsubsection{I/O Lower Bound Result For Composite Codes}
\label{subsubsection: IO lower bound results}

\begin{theorem}
	\label{theorem: general IO lower bound result}
	Assume that a DAG $G(V,E)$ describes an algorithm with $n$ steps. All sub-computations in $n$ steps are corresponding to a multi-step partition of the DAG.
	Given a fast memory of size $S$, to finish the algorithm, the minimum number $Q$ of I/O operations between the fast memory and the slow memory satisfies
	\begin{equation}
		\label{equation: general IO lower bound result}
		Q \geq S \cdot \left(\dfrac{|V|}{T(2S)} -1 \right).
	\end{equation}
\end{theorem}

\begin{proof}
Based on Equations (\ref{equation: HS}) and (\ref{equation: QQ}), we have 
\begin{equation}
\label{equation: QQQ}
Q \geq S \cdot \left(\min_{\{V_{1},\cdots, V_{h}\} \in \mathbb{P}_{2S} }  \dfrac{|V|}{\max_{1\leq i \leq h} |V_{i}|} -1 \right).
\end{equation}
For any $\{V_{1},\cdots, V_{h}\} \in \mathbb{P}_{2S}$, 
Theorem \ref{theorem: upper bound of Vi} directly leads to
$\max_{1\leq i \leq h} |V_{i}| \leq T(2S)$, which, together with Equation (\ref{equation: QQQ}), implies that
$Q \geq S * ({|V|}/{T(2S)} -1)$ is valid.
\end{proof}

Theorem \ref{theorem: general IO lower bound result} concludes how the I/O lower bound of any composite algorithm depends on the upper bounds of $\varphi_j$ and $\psi_j$. It not only gives a general theoretical result, but also presents a lower bound proof method which is to estimate $\varphi_j$ and $\psi_j$ one by one. In addition, although Equation (\ref{equation: general IO lower bound result}) is similar to Equation (\ref{equation: original lower bound}), it is easier to obtain $T$ for a composite code. 

\subsection{I/O Lower Bounds for Direct Convolution}

\begin{figure}
	\centering
	\includegraphics[scale=0.3]{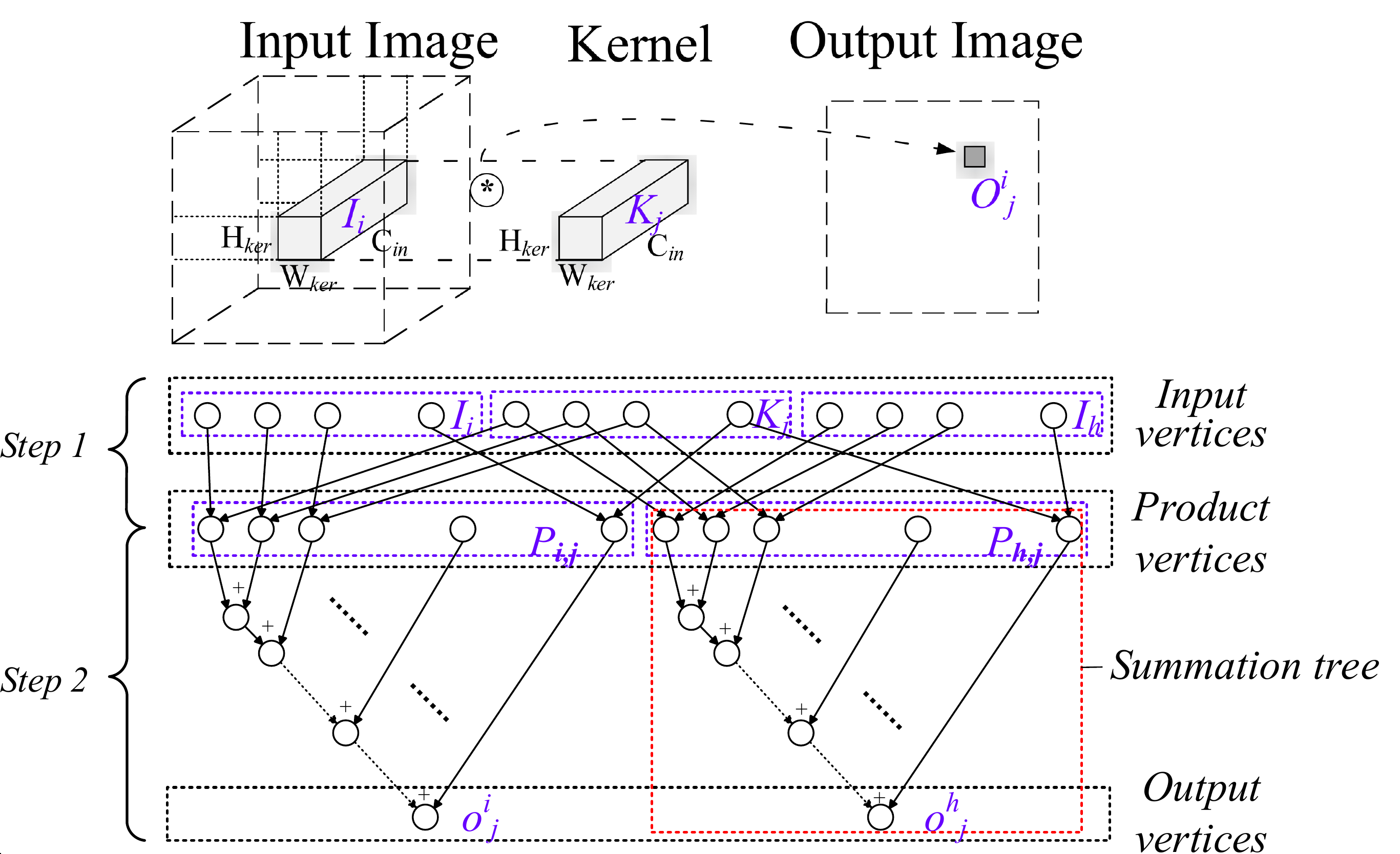}
	\caption{DAG of Direct Convolution.
		\label{fig:DAGofDirectConvolution}}
\end{figure}

Figure \ref{fig:DAGofDirectConvolution} shows a DAG $G(V,E)$ of the direct convolution. It is clear that the direct convolution consists of two steps. The first step is to generate a lot of product terms by using inputs in the input images and kernels. In DAG $G(V,E)$, we call the product vertices as the vertices which are corresponding to the product terms generated by the first step in the direct convolution. Denote $I_i$ as the $i$-th sliding input tensor with the size of $W_{ker} \times H_{ker} \times C_{in}$ from an input image by using a sliding window. Let $K_j$ be the $j$-th kernel whose size is also $W_{ker} \times  H_{ker} \times C_{in}$. For each $I_i$ and $K_j$, the first step generates $W_{ker} H_{ker} C_{in}$ product terms by executing the corresponding element product of $I_i$ with $K_j$. The second step is to sum the product terms generated by $I_i$ and $K_j$ to form one of final outputs based on a summation tree. The summation tree is a sub-DAG with the tree structure in which, except for all input vertices, the in-degree of other vertices is at most two, and all inputs of the tree would be summed together to the only one output (Figure \ref{fig:DAGofDirectConvolution}). After the summation process, the direct convolution is finished. Hence, the multi-step partition of $G(V,E)$ can be written as $G(V,E) = G_1(U_1,E_1) \cup G_2(U_2,E_2)$ where the sub-DAG $G_i(U_i,E_i)$ is corresponding to the $i$-th step of the direct convolution. \begin{lemma}
\label{lemma: the number of vertices on a summation tree}
A summation tree with $k$ input vertices involves $k-2$ internal vertices and $1$ output.
\end{lemma}
\begin{proof}
For a summation tree with $k$ input vertices, the summation of the first two vertices would generate the first internal vertex. Next, we add the first internal vertex with the third input vertex, resulting the second internal vertex of the summation tree. Furthermore, we
continue the process above. After the $(k-2)$-th internal vertex is added with the $k$-th input vertex, the final output would be generated. Hence, there are $(k-2)$ internal vertices and $1$ output vertex on a summation tree. 
\end{proof}

Based on the summation tree structure, we can calculate the total number of internal and output vertices in the DAG of a direct convolution.

\begin{lemma} 
	\label{lemma: the number of vertices in DAG of DC}
	In a DAG of any direct convolution, the total number of internal and output vertices is 
	\begin{displaymath}
		| V_{inter} \cup V_{out} |= (2W_{ker} H_{ker} C_{in}-1) W_{out} H_{out} C_{out}.
	\end{displaymath}
\end{lemma}

\begin{proof}
It is obvious that the DAG of a direct convolution has $W_{out} H_{out} C_{out}$ output vertices. Each output is the summation of the corresponding element product of $I_i$ and $K_j$, where $I_i$ is the $i$-th sliding input tensor and $K_j$ is the $j$-th kernel. $I_i$ and $K_j$ have the same dimension of $W_{ker} \times H_{ker} \times C_{in}$.
	
Firstly, the two tensors $I_i$ and $K_j$ are associated with two input vertex sets. Secondly, after executing the corresponding element product of $I_i$ and $K_j$, we can obtain $W_{ker} H_{ker} C_{in}$ product vertices which are $W_{ker} H_{ker} C_{in}$ of the outputs of $G_1(U_1,E_1)$. Thirdly, Lemma \ref{lemma: the number of vertices on a summation tree} indicates that, in order to sum $W_{ker} H_{ker} C_{in}$ internal results based on the summation tree,
the second step would generate another $W_{ker} H_{ker} C_{in}-2$ internal vertices and $1$ output vertex. Consequently, one output vertex depends on $2 W_{ker} H_{ker} C_{in} -2$ internal vertices.

Since each output vertex is generated independently, no internal vertex would be shared by two different summation trees. Hence, the total number of internal and output vertices in DAG is 
$(2W_{ker} H_{ker} C_{in}-1) W_{out} H_{out} C_{out}.$
\end{proof}

For any two tensors $a$ and $b$ with the same dimension, denote $a \circledast b$ as the summation of all corresponding element products of $a$ and $b$. By this notation, the $i$-th output vertex $O^{j}_{i}$ in the $j$-th output channel can be represented as $O^{j}_{i} = I_i \circledast K_j$ (Figure \ref{fig:DAGofDirectConvolution}). Before estimating the upper bound $T(S)$, we denote $R$ as the maximum reuse number of each input (element) in an input image by different silding windows, whose value is
\begin{equation}
\label{equation: definition of R}
R = \dfrac{W_{ker}H_{ker}}{\mu^2},
\end{equation}
where $\mu$ is the stride size. We will see that the $T(S)$ relies on $R$.

\begin{lemma}
\label{lemma: upper bound of ViU1 for DC}
In a direct convolution, $\psi_1 = \varphi_1$ is valid for the first step. Futhermore, for any positive integer $h_1$, $\varphi_1(h_1) \leq 2S\sqrt{R h_1}$.
\end{lemma}
\begin{proof}
Let $U$ be any vertex set whose dominator set $D$ and minimun set $M$ contain at most $S$ vertices. Suppose $|D \cap U_{1}| \leq h_1 $. In order to estimate $\varphi_1$ and $\psi_1$, we consider how many vertices in $U \cap U_1$ can be generated by $D \cap U_1$. Since $G_1(U_1,E_1)$ has no internal vertices, all vertices generated by $D \cap U_{1}$, can be used as the inputs of $G_2(V_2,E_2)$. Hence, $\psi_1 = \varphi_1$ is valid. 

Since $|M| \leq S$ and the internal vertex sets of different summation trees are disjoint from each other, $U$ can have nonempty intersections with internal vertex sets of at most $S$ summation trees, and each intersection has at least one distinct vertex in the minimum set. For any $\widetilde{O} = \{O^{j_1}_{i_1}, O^{j_2}_{i_2}, \cdots, O^{j_S}_{i_S} \}$ with $S$ output vertices, each vertex can be formed by $O^{j_k}_{i_k} = I_{i_k} \circledast K_{j_k}$ (Figure \ref{fig:DAGofDirectConvolution}). Without loss of generality, we assume that each $I_i$ in the set $I^1 =\{I_{i_1},I_{i_2},\cdots, I_{i_{k_0}} \}$ has at least $\sqrt{R h_1}$ entries in $D \cap U_1$, while each $I_{i'}$ in the reset $I^2 = \{I_{k_0+1},I_{k_0+2},\cdots, I_{i_{S}} \}$ has at most $\sqrt{R h_1}$ entries in $D \cap U_1$.

On one hand, for the set $I^1 = \{I_{i_1},I_{i_2},\cdots, I_{i_{k_0}} \}$, we can prove that $k_0 \leq \sqrt{R h_1}$. In the following, we verify this fact by reductio ad absurdum. Assume that $k_0 > \sqrt{R h_1}$. Since each $I_i \in I^1$ has at least $\sqrt{RS}$ entries in $D \cap U_1$, the set $I^1$ involves at least $k_0\sqrt{RS}$ vertices belonging to $D \cap U_1$, while some of these vertices may be the same. As each input vertices can be reused at most $R$ times, there exist at least $k_0\sqrt{RS}/R$ independent vertices in $D \cap U_1$, and $ k_0\sqrt{RS}/R > h_1$, which contradicts with the fact $|D \cap U_1| \leq h_1$. Thus, the assumption is not valid, and $k_0 \leq \sqrt{R h_1}$ is valid. Due to $|D| \leq S$, $\cup^{S}_{k=1}K_{j_k}$ involves no more than $S$ vertices in $D$. Hence, the vertices in $D \cap (\cup^{k_0}_{k=1}I_{i_k})$ and $D \cap (\cup^{S}_{k=1}K_{j_k})$ can generate at most $S\sqrt{R h_1}$ products for $O^{j_k}_{i_k} (k=1,2,\cdots, k_0)$. On the other hand, for the set $I^2 = \{I_{k_0+1},I_{k_0+2},\cdots, I_{i_{S}} \}$, since each $I_i \in I^2$ has at most $\sqrt{R h_1}$ entries in $D$, at most $S\sqrt{R h_1}$ products can be formed by the vertices in $D \cap (\cup^{i_S}_{k=k_0+1}I_{i_k})$ and $D \cap (\cup^{S}_{k=1}K_{j_k})$. In conclusion, $D \cap U_1$ can generate at most $2S\sqrt{R h_1}$ in $U \cap U_1$. This means that $\varphi(h_1) \leq 2S\sqrt{R h_1}$ is valid.
\end{proof}

\begin{lemma}
\label{lemma: upper bound of ViU2 for DC}
In a direct convolution, for any positive integer $h_2$, $\varphi_2(h_2) \leq h_2-1$.
\end{lemma}
\begin{proof}

Assume that a vertex set $U$ has a dominator set $D$ and a minimun set $M$ satisfying $|D \cap U_{2}| + | \varTheta(D) \cap \widetilde{O}_{1} | \leq h_2 $, $|D| \leq S$ and $|M| \leq S$. To deduce the upper bound of $\varphi_2$, we only need to consider the vertices in $U \cap U_2$ which can be formed by $D \cap U_2$ and $\varTheta(D) \cap \widetilde{O}_{1}$ . As $|D \cap U_{2}| + | \varTheta(D) \cap \widetilde{O}_{1} | \leq h_2 $, at most $h_2$ vertices would be the inputs of summation trees. By Lemma \ref{lemma: the number of vertices on a summation tree}, such $h_2$ vertices can generate at most $h_2-1$ internal vertices in the intersections of $V_i$ with $S$ summation trees. Therefore, we get $\varphi_2(h_2) \leq h_2-1$.
\end{proof}

Lemma \ref{lemma: upper bound of ViU1 for DC} and Lemma \ref{lemma: upper bound of ViU2 for DC} lead to an estimation of $T(S)$ directly.

\begin{lemma}
	\label{lemma: upper bound of TS for DC}
	For a direct convolution, $T(S)\leq 4 S \sqrt{RS} + S - 1$.
\end{lemma}
\begin{proof}
By the definition of $T$, we deduce
\begin{displaymath}
T(S) \leq  S + \max_{k_1 + k_2 \leq S} \{ 2S\sqrt{Rk_1} + (k_2 + 2S\sqrt{Rk_1} -1) \} \leq S + 4 S\sqrt{RS} -1,
\end{displaymath}
where the final equality holds if and only if $k_1 = S$ and $k_2 =0$.
\end{proof}

\begin{theorem}
	\label{theorem: lower bound results of DC}
	The I/O lower bound of a direct convolution (DC) is
	\begin{equation}
		\label{equation: communication lower bound of DC}
		Q_{lower~~DC} =  \Omega \left( \dfrac{W_{ker} H_{ker} C_{in} W_{out} H_{out} C_{out}}{4\sqrt{2R S}} \right).
	\end{equation}
\end{theorem}

\begin{proof}
By Theorem \ref{theorem: general IO lower bound result}, Lemma \ref{lemma: the number of vertices in DAG of DC} and Lemma \ref{lemma: upper bound of TS for DC}, we have
\begin{displaymath}
Q \geq 
\dfrac{(2W_{ker} H_{ker} C_{in}-1) W_{out} H_{out} C_{out}}{8\sqrt{2R S} + 2 - 1/S} - \dfrac{1}{S},
\end{displaymath}
which implies that (\ref{equation: communication lower bound of DC}) is valid.
\end{proof}

It is worth mentioning that the derived lower bound is in
the form of $\Omega$ instead of a precise value. It provides the
asymptotic relation between the data movement and
the fast memory capacity when the problem scale is large
enough.

\subsection{I/O Lower Bounds for Winograd Algorithm}

\begin{figure}
\centering
\includegraphics[scale=0.32]{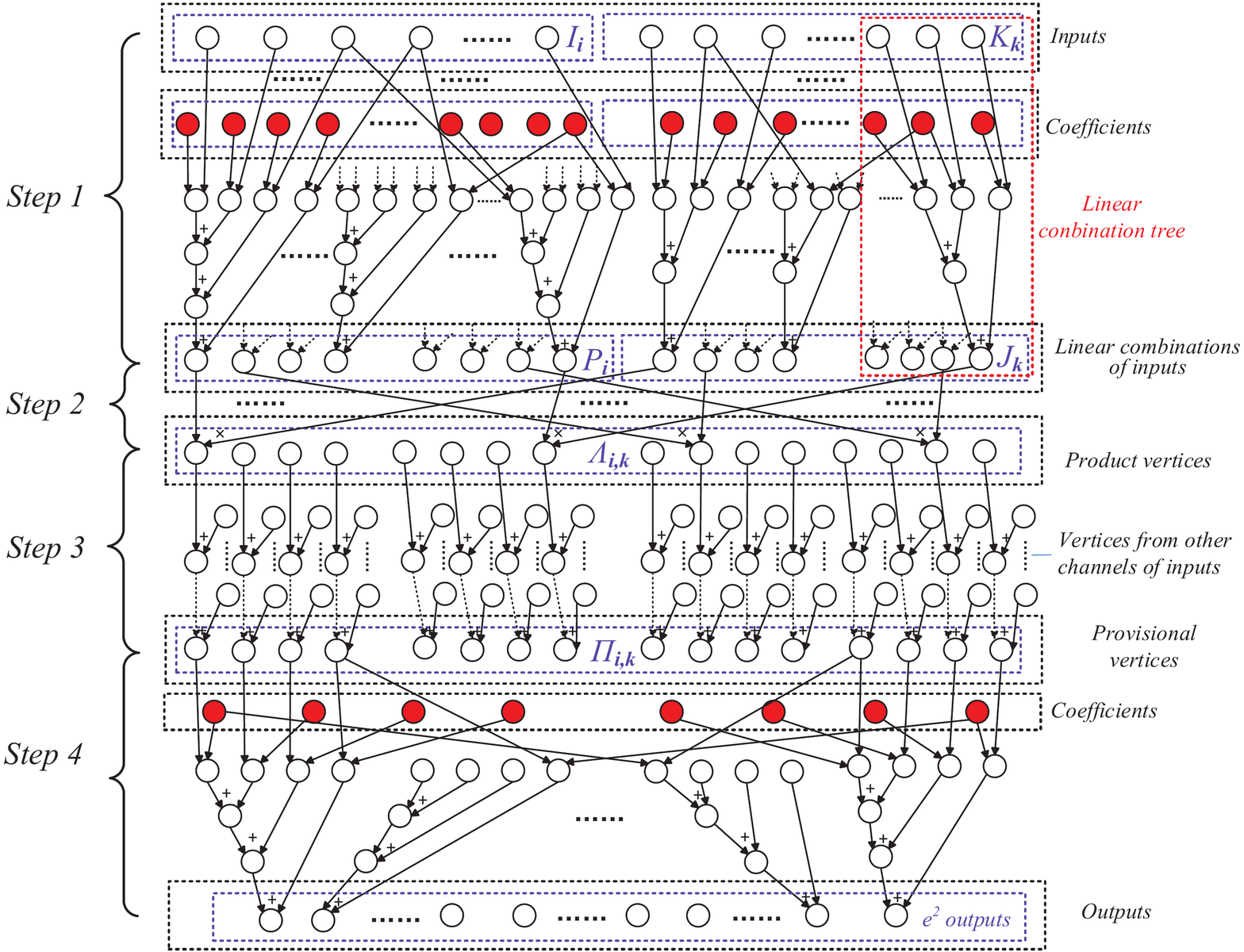}
\caption{DAG of Winograd Algorithm.
\label{fig:DAGofWinogradAlgorithm}}
\end{figure}

In Winograd algorithm, since the size of three transformation matrices $A$, $B$ and $L$ is small, we assume that they can be always stored in fast storage, and their volume can be ignored compared with the size $S$ of the fast memory. Futheremore, as Winograd algorithm requires $W_{ker}=H_{ker}$, we denote $r$ as $W_{ker}$ or $H_{ker}$ briefly. As mentioned in Section \ref{section: Winograd Algorithm}, Winograd algorithm decompsites the output matrix on each given channel of a output image into several sub-domains whose size is $e \times e$. Every $e^2$ outputs on each subdomain are calculated simultaneously using $F(e, r)$. Figure \ref{fig:DAGofWinogradAlgorithm} is a DAG $G(V,E)$ of Winograd algorithm which consists four steps. In Figure \ref{fig:DAGofWinogradAlgorithm}, the red vertices represent the elements of three transformation matrices which would not involve any I/O procedure. The first step is a tensor conversion process. Using a sliding window, a sliding input tensor $I_i$ with the size of $(e+r-1) \times (e+r-1) \times C_{in}$ is chosen from an input image, where is a positve integer. The first step transforms $I_i$ into $P_i$ by multiplying the transformation matrix $B$ with the $(e+r-1) \times (e+r-1)$ matrix of $I_i$ at each fixed channel. The size of $P_i$ is also $(e+r-1) \times (e+r-1) \times C_{in}$, Similarly, by using $L$ and the $k$-th kernel $K_k$, another tensor $J_k$ can be formed with the size of $(e+r-1) \times (e+r-1) \times C_{in}$. The second step is to execute the element-wise multiplication of $P_i$ with $J_k$, which results in a new tensor $\varLambda_{i,k}$. The third step is to sum the elements of $\varLambda_{i,k}$ along the channel direction through the summation tree. A $(e+r-1) \times (e+r-1)$ matrix $\varPi_{i,k}$ is obtained. The final step is to use the matrix $A$ to transfer $\varPi_{i,k}$ into a $e \times e$ matrix which would be $e^2$ outputs on the $i$-th sub-domain at the $k$-th channel of an output image. It is worth mentioning that both the tensor conversion of the first step and the matrix conversion of the fourth step can be realized through a linear combination tree (Figure \ref{fig:DAGofWinogradAlgorithm}). Similar to the summation tree, a linear combination tree is a sub-DAG with the tree structure in which the in-degree of internal and output vertices is at most two. All inputs of linear combination tree are mutiplied with different coefficients respectively at first, and then summed together to the only one output.

\begin{lemma}
\label{lemma: the number of vertices on a linear combination process}
A linear combination tree with $k$ input vertices involves $2k-2$ internal vertices and $1$ output.
\end{lemma}
\begin{proof}
For a linear combination process with $k$ inputs, all inputs are multiplied respectively by different coefficients that are always stored at fast memory. This multiplication could result in $k$ interal vertices. Furthermore, $k$ internal vertices are summed together through a summation tree. From Lemma \ref{lemma: the number of vertices on a summation tree}, another $k$ interal vertices and $1$ output node are formed. Hence, the linear combination tree has $2k-2$ internal vertices and $1$ output.
\end{proof}

\begin{lemma} 
\label{lemma: the number of vertices in DAG of WA}
In a DAG of any winograd algorithm, the total number of internal and output vertices is 
\begin{displaymath}
| V_{inter} \cup V_{out} |= O \left( \dfrac{2W_{out}H_{out}C_{out}C_{in}(e+r-1)^4}{e^2} \right).
\end{displaymath}
\end{lemma}
\begin{proof}
A winograd algorithm has $W_{out}H_{out}C_{out}$ output vertices. Every $e^2$ outputs are calculated at once by using a $(e+r-1) \times (e+r-1) \times C_{in}$ input tensor $I_i$ and a $r \times r \times C_{in}$ kernel $K_k$. 
At the first step, $I_i$ and $K_k$ generate $P_i$ and $J_k$ respectively.  $P_i$ and $J_k$ have the same dimension of $(e+r-1) \times (e+r-1) \times C_{in}$. Each vertex in $P_i$ is formed through a linear conbination tree with $(e+r-1)^2$ inputs from the input matrix at some channel of $I_i$. By Lemma \ref{lemma: the number of vertices on a linear combination process}, $(2(e+r-1)^2 -1)(e+r-1)^2C_{in}$ vertices are generated. Similarly, each vertex in $J_k$ is formed through a linear conbination tree with $r^2$ weights in $K_k$, which involves $(2r^2-1)(e+r-1)^2C_{in}$ vertices. In the second step, the corresponding element-wise multiplication of $P_i \odot J_j$ forms $(e+r-1)^2 C_{in}$ vertices further. The third step is to add the elements in $\varLambda_{i,k}$ ($\varLambda_{i,k}=P_i \odot J_j$) along the channel direction to deduce a matrix $\varPi_{i,k}$, which would generate $(C_{in}-1)(e+r-1)^2$ vertices through $(e+r-1)^2$ different summation trees (by Lemma \ref{lemma: the number of vertices on a summation tree}). Finally, the fourth step is to obtain $e^2$ outputs on the $i$-th sub-domain at the $k$-th channel of an output image. Since each vertex from the $e^2$ outputs is generated through a linear conbination tree with the inputs of all $(e+r-1)^2$ elements in $\varPi_{i,k}$. By Lemma \ref{lemma: the number of vertices on a linear combination process} again, the fourth step involves $(2(e+r-1)^2-1)e^2$ vertices. Since each $e^2$ output vertices are generated independently, the total number of internal and output vertices in DAG is $O(2W_{out}H_{out}C_{out}C_{in}(e+r-1)^4/e^2)$.
\end{proof}

Let $\{G_1(U_1,E_1), G_2(U_2,E_2), G_3(U_3,E_3), G_4(U_4,E_4)\}$ be the multi-step partition of $G(V,E)$. Since $P_i$ and $J_k$ are obtained indenpendently, we further divide the sub-DAG $G_1(U_1,E_1)$ into two small DAGs $G_{1,2}(U_{1,1},E_{1,1})$ and $G_{1,2}(U_{1,2},E_{1,2})$ where $G_{1,1}$ and $G_{1,2}$ are corresponding to the generation process of $P_i$ and $J_k$ respectively. As $r=e \pm 1$ is satisfied in Winograd algorithm, we can assume that $1/2 \leq r/e \leq 2$ in our estimation for $\varphi_i$ and $\psi_i$ ($1 \leq i \leq 4$).

\begin{lemma}
\label{lemma: upper bound of varphi1 for WA}
In Winograd algorithm, for any positive integer $h_1$, 
\begin{equation}
\label{equation: upper bound of varphi1 for WA}
\varphi_1(h_1) \leq \dfrac{6h_1(e+r-1)^4}{er} ~\hbox{and}~ \psi_1(h_1) \leq \dfrac{3h_1(e+r-1)^2}{er}.
\end{equation}
\end{lemma}
\begin{proof}
Let $U$ be any vertex set whose dominator set $D$ and minimun set $M$ contain at most $S$ vertices. Assume that $|D \cap U_{1}| \leq h_1 $. In order to estimate $\varphi_1$ and $\psi_1$, we consider the vertices in $U \cap U_1$ generated by $D \cap U_{1, 1}$ and $D \cap U_{1, 2}$ respectively. Denote $k_1 = |D \cap U_{1, 1}|$ and $k_2=|D \cap U_{1, 2}|$ respectively. As $U_{1,1}$ is disjoint with $U_{1,2}$, we get $k_1 + k_2 \leq h_1$. On one hand, in sub-DAG $G_{1,1}$, every $(e+r-1)^2$ vertices in $D \cap U_{1, 1}$ are used as the inputs of $(e+r-1)^2$ linear conbincation trees. By Lemma \ref{lemma: the number of vertices on a linear combination process}, with $(e+r-1)^2$ inputs, each linear conbincation tree can generate $2(e+r-1)^2-2$ internal vertices and $1$ output. Hence, every $(e+r-1)^2$ vertices from $D \cap U_{1, 1}$ would form at most $(e+r-1)^2 \cdot (2(e+r-1)^2-2)$ internal vertices and $(e+r-1)^2$ outputs. Furthermore, since the reuse number of each input vertex is $(e+r-1)^2/e^2$, $D \cap U_{1, 1}$ can generate at most $2k_1(e+r-1)^4/e^2$ vertices in which there are $k_1(e+r-1)^2/e^2$ vertices in the output set of $G_{1,1}$. On the other hand, in sub-DAG $G_{1,2}$, every $r^2$ vertices in $D \cap U_{1, 2}$ are used as the inputs of $(e+r-1)^2$ conbincation trees, while any vertex in $D \cap U_{1, 2}$ would not be reused. Similar to the discussion above, it is clear that $D \cap U_{1, 2}$ could generate at most $2k_2(e+r-1)^4/r^2$ vertices in which $k_2(e+r-1)^2/r^2$ vertices are in the output set of $G_{1,2}$. Since $1/2 \leq r/e \leq 2$, we have $\varphi_1(h_1) \leq \max_{k_1 +k_2 \leq h_1} 2(e+r-1)^4 (\dfrac{k_1}{e^2} + \dfrac{k_2}{r^2}) \leq \dfrac{6 h_1(e+r-1)^4}{er}$ and $\psi_1(h_1) \leq \max_{k_1 +k_2 \leq h_1} 
(e+r-1)^2 (\dfrac{k_1}{e^2} + \dfrac{k_2}{r^2}) 
\leq \dfrac{3 h_1(e+r-1)^2}{er}$.
\end{proof}

\begin{lemma}
\label{lemma: upper bound of varphi2 for WA}
In Winograd algorithm, $\psi_2 = \varphi_2$ is valid for the second step. Futhermore, for any positive integer $h_2$,
\begin{equation}
\label{equation: upper bound of varphi2 for WA}
\varphi_2(h_2) \leq h_2 \sqrt{h_2} + \dfrac{(e+r-1)^2S}{e^2} \sqrt{h_2}.
\end{equation}
\end{lemma}

\begin{proof}
In the sub-DAG $G_2(U_2, E_2)$, there is no internal vertices. Hence, $\psi_2 = \varphi_2$ is valid. Assume that a vertex set $U$ has a dominator set $D$ and a minimun set $M$ satisfying $|D \cap U_{2}| + |\varTheta(D) \cap \widetilde{O}_{1} | \leq h_2 $, $|D| \leq S$ and $|M| \leq S$. Since $M$ has no more than $S$ vertices, $U$ can have nonempty intersections with internal vertex sets of at most $S$ different linear conbination trees of the fourth step. We note that every $(e+r-1)^2$ outputs of summation trees in the third step, would be used as inputs of $e^2$ conbination trees. Due to $|M| \leq S$ again, $U$ must intersect with at most $S(e+r-1)^2/e^2$ independent summation trees of the third step. To estimate the upper bound of $\varphi_2$, we only need to consider the vertices in $U \cap U_2$ which are generated by $D \cap U_2$ and $\varTheta(D) \cap \widetilde{O}_{1}$ for at most $S(e+r-1)^2/e^2$ disjoint summation trees in the third step. Similar to the proof of Lemma \ref{lemma: upper bound of ViU1 for DC}, we can deduce that
$\varphi_2(h_2) \leq h_2 \sqrt{h_2} + (e+r-1)^2S \sqrt{h_2}/e^2$.
\end{proof}

\begin{lemma}
\label{lemma: upper bound of varphi3 for WA}
In Winograd algorithm, for any positive integer $h_3$, $\varphi_3(h_3) \leq h_3-1$ and $\psi_3(h_3) \leq \min \{h_3/2, S(e+r-1)^2/e^2\}$.
\end{lemma}
\begin{proof}
By Lemma \ref{lemma: the number of vertices on a summation tree}, it is clear that $\varphi_3(h_3) \leq h_3-1$. Based on the discussion in the proof of Lemma \ref{lemma: upper bound of varphi2 for WA}, for any $U$ whose minimum set has at most $S$ vertices, $U$ must have nonempty intersections with internal vertex sets of at most $S(e+r-1)^2/e^2$ independent summation trees of the third step. Hence, $\psi_3(h_3) \leq S(e+r-1)^2/e^2$ is valid for any $h_3 \geq 1$. Futhermore, In the third step, none of the input and internal vertices in one summation tree appears as a vertex in another. Hence, at least two vertices in fast memory are needed to form one output vertex of a summation tree. This means $\psi(h_3) \leq h_3/2$. Consequently, $\psi(h_3) \leq \min \{h_3/2, S(e+r-1)^2/e^2\}$ is valid.
\end{proof}

\begin{lemma}
\label{lemma: upper bound of varphi4 for WA}
In Winograd algorithm, for any positive integer $h_4$, $\varphi_4(h_4) \leq \min \{(2h_4-1)e^2, (2(e+r-1)^2-1) S\} $.
\end{lemma}
\begin{proof}
Let $U$ be any vertex set whose dominator set $D$ and minimun set $M$ satisfy $|D \cap U_{4}| + |\varTheta(D) \cap \widetilde{O}_{3} | \leq h_4$, $|D| \leq S$ and $|M| \leq S$. Since each input of $U_4$ is used as an input for $e^2$ linear conbination trees, Lemma \ref{lemma: the number of vertices on a linear combination process} leads to $\varphi_4(h_4) \leq e^2 (2h_4-1)$. In the fourth step, each linear combination tree has $(e+r-1)^2$ inputs, and at most $S$ linear combination trees have nonempty intersections with internal vertex sets of $U$. By Lemma \ref{lemma: the number of vertices on a linear combination process}, we have $\varphi_4(h_4) \leq (2(e+r-1)^2-1) S$, Therefore, $\varphi_4(h_4) \leq \min \{e^2 h_4-1, (2(e+r-1)^2-1) S\}$.
\end{proof}

Based on the upper bounds of $\varphi_i$ and $\psi_i$ ($1 \leq i \leq 4$), it is easy to estimate $T$ for Winograd algorithm.

\begin{lemma}
\label{lemma: upper bound of TS for WA}
For Winograd algorithm,
\begin{equation}
\label{equation: upper bound of TS for WA}
T(S) = O \left( 2\dfrac{(e+r-1)^3 }{er} S\sqrt{S} + \dfrac{6(e+r-1)^2}{er} S\right).
\end{equation}
\end{lemma}
\begin{proof}
Set $h(k_1, k_2) = k_2 + 3k_1(e+r-1)^2/er$ for any integer $k_1$ and $k_2$. Denote $T_1$ and $T_2$ as $T_1(k_1) = 6k_1 (e+r-1)^4/er$ and
\begin{displaymath}
T_2(k_1, k_2) = h(k_1, k_2) \sqrt{h(k_1, k_2)} + \dfrac{(e+r-1)^2}{e^2}S\sqrt{h(k_1, k_2)}.
\end{displaymath}
By Lemma \ref{lemma: upper bound of varphi1 for WA} - Lemma \ref{lemma: upper bound of varphi4 for WA}, we can deduce
\begin{equation}
T(S) \leq S + T_{1}(S) +T_2(S,0) + (e+r-1)^2 \left(\dfrac{1}{e^2}+2 \right) S.
\end{equation}
Therefore, it is easy to check that $T_2(S,0) \leq 2\dfrac{(e+r-1)^3 }{er} S\sqrt{S}$. In conclusion, the equation (\ref{lemma: upper bound of TS for WA}) is valid.
\end{proof}

So far, the lower bound of I/O complexity of Winograd algorithm can be established.

\begin{theorem}
\label{theorem: lower bound results of IM}
The communication lower bound of Winograd algorithm (WA) is
\begin{equation}
\label{equation: communication lower bound of WA}
Q_{lower~~WA} = \Omega \left( \dfrac{W_{out} H_{out} C_{out}C_{in}(e+r-1)r}{e\sqrt{S}} \right).
\end{equation}
\end{theorem}

\begin{proof}
By Theorem \ref{theorem: general IO lower bound result}, Lemma \ref{lemma: the number of vertices in DAG of WA} and Lemma \ref{lemma: upper bound of TS for WA},
we have
\begin{displaymath}
Q \geq S \cdot \left(
\dfrac{2 W_{out} H_{out} C_{out}C_{in}(e+r-1)^4}{e^2 T(S)} -1 \right),
\end{displaymath}
which implies that (\ref{equation: communication lower bound of WA}) is valid.
\end{proof}
\section{Near I/O-Optimal Strategy}
\label{section:I/O Optimal Dataflow}

\subsection{Methodology for Near I/O-Optimal Strategy}
\label{subsection:Methodology for Near I/O-Optimal Strategy}

In the proposed general I/O lower bound theory, the highest order term in I/O lower bound result (\ref{equation: general IO lower bound result}) must be determined by some $\varphi_i$ due to the definition (\ref{equation: upper bound2 of Vi}) of $T$. Specifically, for the direct convolution, the maximum vertex generation function $\varphi_2$ for the last step determines the highest order term in I/O lower bound (Equation (\ref{equation: communication lower bound of DC})). For Winograd algorithm, the highest order term in I/O lower bound (Equation (\ref{equation: communication lower bound of WA})) comes from $\varphi_3$ for the third step, rather than $\varphi_4$ for the last step. As the highest order term in I/O lower bound result represents the main part of I/O number, the related $\varphi_i$ points to the major process which involves the most I/O operations. 

By the function $\varphi_i$ which determines the highest order term in I/O lower bound result of a composite algorithm, we are able to find which data should be fully reused in the on-chip memory, and minimize the number of I/O operations during the $i$-th step of the composite algorithm. In detail, for the direct convolution, $\varphi_2$ which determines the highest order term in Equation (\ref{equation: communication lower bound of DC}), indicates that minimizing the number of I/O operations needs to maximize the output data reuse. For Winograd algorithm, $\varphi_3$ inspires us to maximize the data reuse of two temporary arrays which are involved during the third step.

After determining which data reuse should be exploited, the dataflow strategy can be designed to maximize the reuse of such data. In the following, we propose different schedules for the direct and Winograd convolutions by maximizing the reuse of output data and two temporary arrays respectively.





\subsection{Dataflow Design for Direct Convolution}

\begin{figure}[htbp]
	\centering
	\includegraphics[scale=0.35]{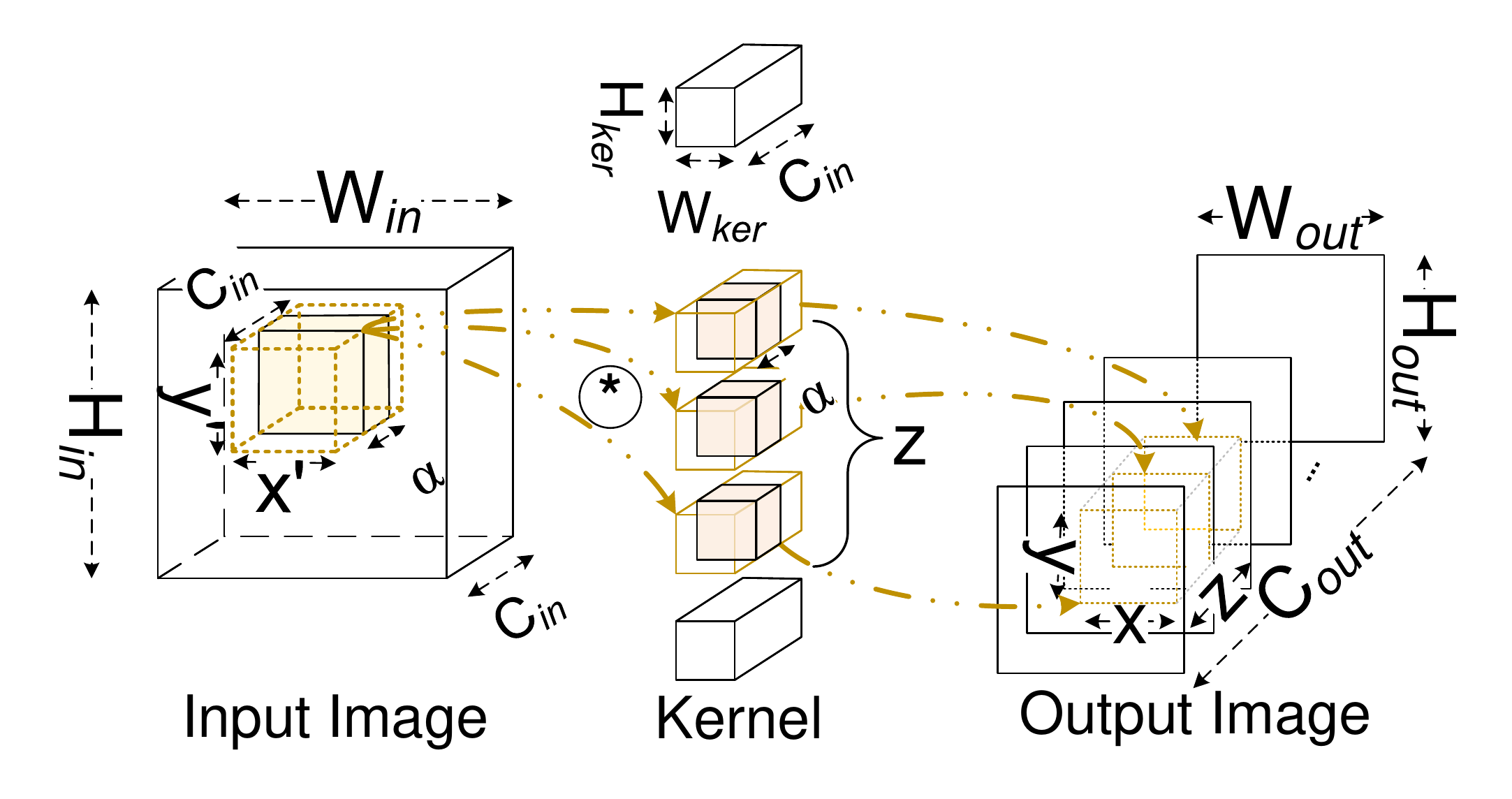}
	\caption{I/O Optimal Dataflow for Direct Convolution.}
	\label{fig:ImplemenationForDirectConvolution}
\end{figure}

For the direct convolution, the highest order term in I/O lower bound (Equation (\ref{equation: communication lower bound of DC})) comes from $\varphi_2$ for the last step. $\varphi_2$ indicates that the output data reuse should be fully exploited, which implies that we need to use the least inputs to produce the most outputs. Hence, the dataflow design should assign most of the effective on-chip memory to portions of outputs. Figure \ref{fig:ImplemenationForDirectConvolution} shows a sub-block of the output image with the dimension of $x \times y \times z$. Based on the fundamental principle above, to reach the minimum off-chip memory access, we tend to choose $xyz \approx S / N_{p}$ where $N_p$ is the total number of active processors.

To compute the output sub-block $x \times y \times z$, we need the inputs in the corresponding $x' \times y'$ locations from all input channels (the yellow sub-block in an input image) and $z$ kernels associated with the partial output channels (the yellow kernels), as shown in Figure \ref{fig:ImplemenationForDirectConvolution}. Since the on-chip memory is limited and tends to be used for storing the most outputs, it is necessary to load the required inputs and kernels by a series of stages, rather than at a time. During each stage, a portion of inputs $x' \times y' \times \alpha$ (the black sub-block) and the corresponding weights $H_{ker} \times W_{ker} \times \alpha$ of $z$ kernels are loaded to the on-chip memory (Figure \ref{fig:ImplemenationForDirectConvolution}). Since each input of the $i$-th channel only can be reused by the weights of the $i$-th channel, rather than other channels. In order to put the larger output sub-block in the limited on-chip memory, we set $\alpha=1$, which indicates that our dataflow design is to load a $x' \times y'$ tile with a fixed channel index firstly and then slide the tile along the channel direction.

After loading a $x' \times y'$ input tile and the corresponding weights of $z$ kernels into the on-chip memory, a partial sum can be performed on the output sub-block. To update the whole output sub-block, we continuously slide the $x' \times y'$ input tile along the channel direction, and load the corresponding inputs and weights (in the yellow blocks), and perform partial updates. Consequently, updating each output sub-block only needs to load the required inputs and weights from the off-chip memory to on-chip memory exactly once. Meanwhile, different output sub-blocks are updated by $N_{p}$ processors in parallel. 

In our dataflow design, there are $(W_{out}H_{out}C_{out})/(xyz)$ output sub-blocks in total. To update each sub-block, we need $x'y'C_{in}$ inputs from an input image and $W_{ker} H_{ker} C_{in}z$ weights from $z$ kernels. As $R = W_{ker}H_{ker} / \mu^2$, $x' \approx \mu x$ and $y' \approx \mu y$, the I/O volume for reading data is
\begin{eqnarray}
Q_{DC~reading} \approx \frac{H_{out} W_{out} C_{out}}{x y z} \left( H_{ker} W_{ker} C_{in} (z+\frac{x y}{R}) \right) \nonumber \\
\geq H_{out} W_{out} C_{out} H_{ker} W_{ker} C_{in} 
\left( 2 \sqrt{\frac{1}{Rxyz}} \right),
\end{eqnarray}
where the final equality holds if and only if $xy=Rz$. By the fact $R = W_{ker}H_{ker}/\mu^2$, $x'=\mu x$ and $y'=\mu y$ again, the requirement of $xy=Rz$ leads to $x' y' = z W_{ker}H_{ker}$, which determines the optimal size of each $x' \times y'$ tile. Further, the I/O volume for storing outputs is $W_{out}H_{out}C_{out}$. When we choose $xyz \approx S/N_{p} $ and $xy=Rz$, the total I/O volume is 
\begin{equation}
\label{equation: Q of actual DC}
Q_{DC} \approx \frac{2 H_{out} W_{out} C_{out} H_{ker} W_{ker} C_{in}}{\sqrt{RS/N_p}} + H_{out} W_{out} C_{out}. 
\end{equation}
If $N_{p} = 1$ and $\frac{H_{ker} W_{ker} C_{in}}{\sqrt{SR}} \gg 1$ which is easily satisfied in CNN applications due to $S$ usually being equal to or less than KB level, $Q_{DC}$ reaches the I/O lower bound (Theorem \ref{theorem: lower bound results of DC}). This fact indicates that, sequentially executing the dataflow and assigning most of the effective on-chip memory to the outputs can reach the minimum off-chip memory access. Otherwise, if we perform the dataflow in parallel, The equation (\ref{equation: Q of actual DC}) means that fully utilizing the on-chip memory owned by each processor to produce the partial sum could maximize the output data reuse and reduce the data transmission in the memory hierarchy.

In order to view the proposed dataflow at a high level, we conclude the details of this design as follows:

\begin{itemize}
	\item The input data reuse is fully considered. In fact, one input is reused by weights of $z$ kernels, and one weight is reused by $x \times y$ outputs. On the other hand, one input is also reused by at most $R$ sliding windows on each $x' \times y'$ tile.
	
	\item The output data reuse is fully exploited. In fact, the partial sum can always stay in the on-chip memory during the update process, and they are just written back to the off-chip memory only once. To make sure the larger output sub-block can be loaded in the on-chip memory, the optimal tiling is designed to slide the $x' \times y'$ tile along the channel direction, which reveals that the loading of inputs along the width and height directions should be considered prior to the channel direction.
	
	\item In order to achieve the I/O lower bound, the $x \times y \times z$ output sub-block needs to satisfy $xy=Rz$, which is called as the optimality condition in this work. Under this condition, $x' y' = z W_{ker}H_{ker}$, which determines the optimal size of each input tile.
\end{itemize}

\subsection{Dataflow Design for Winograd Algorithm}

\begin{figure}[htbp]
	\centering
	\includegraphics[scale=0.26]{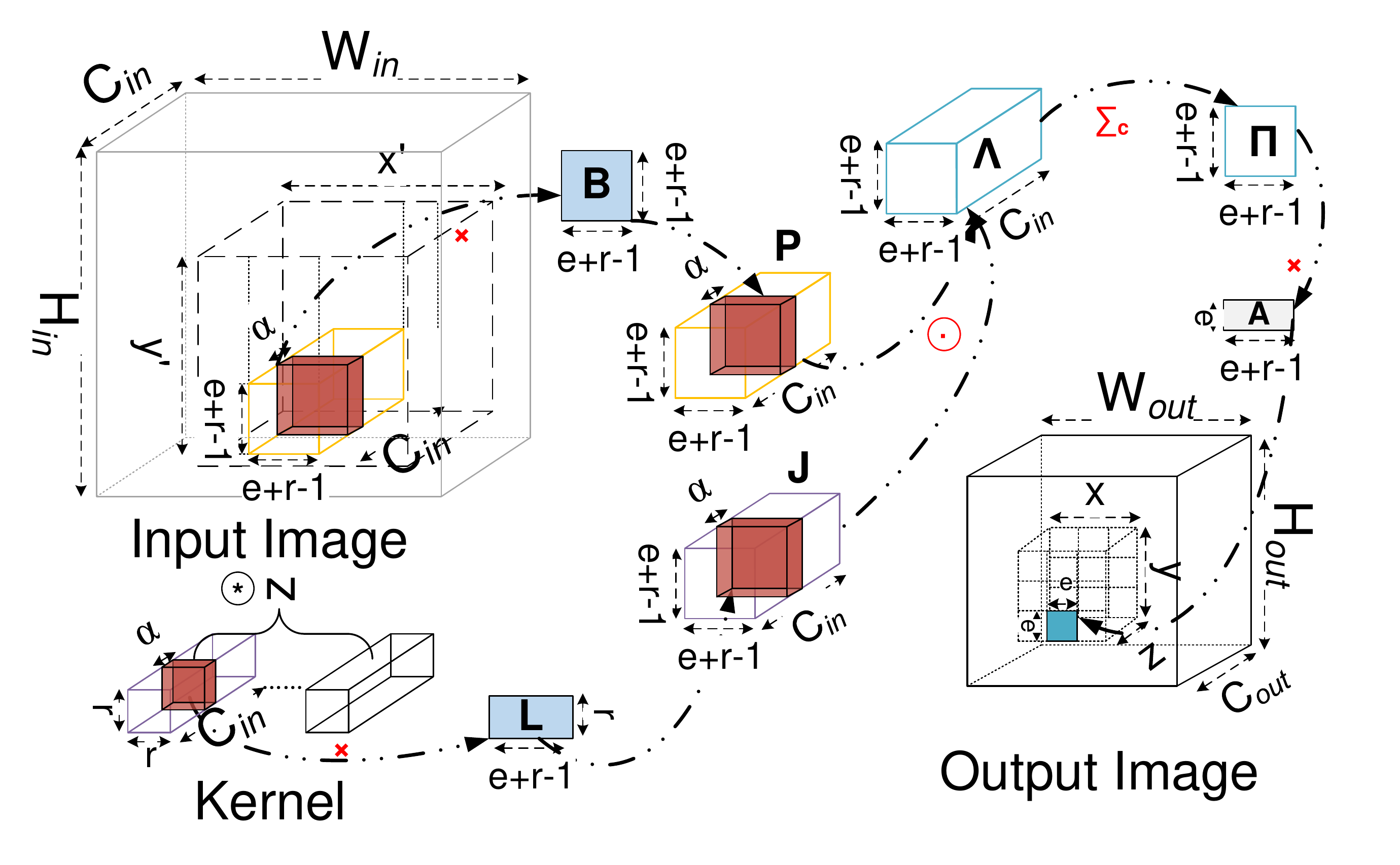}
	\caption{I/O Optimal Dataflow for Winograd Algorithm.}
	\label{fig:ImplemenationForWinogradAlgorithm}
\end{figure}

Similar to the analysis in the dataflow design for direct convolution,  $\varphi_3$ determining the highest order term in I/O lower bound of Winograd algorithm (Equation (\ref{equation: communication lower bound of WA})), leads us to maximize the data reuse of temporary arrays involved during the third step. 

To compute each $x \times y \times z$ output sub-block, Winograd algorithm needs to partition further sub-block into $xy/e^2$ smaller sub-blocks each of which has the size of $e \times e \times z$. Each $e \times e \times z$ small sub-block is computed by using the corresponding $(e+r-1) \times (e+r-1)$ locations from all input channels of the input images (i.e., the yellow block in the input image) and $z$ kernels associated with the partial output channels (Figure \ref{fig:ImplemenationForWinogradAlgorithm}), which are loaded into on-chip memory by a series of stages due to the limited on-chip memory. Based on the same discussion in the dataflow design, each stage loads a $(e+r-1) \times (e+r-1)$ input tile at an input channel (which means $\alpha = 1$) and the corresponding $r^2$ weights at the same channel of a kernel, and then produce a partial sum $\varLambda$ (Figure \ref{fig:ImplemenationForWinogradAlgorithm}). We allocate two $(e+r-1) \times (e+r-1)$ temporary arrays in the on-chip memory for the summation of all partial sums along the channel direction. The first array is used to save the last summation result, and the second one is for the generation of a new partial sum. When a new partial sum is created in the second array, it would be added to the first array. After collecting all partial sums along the channel direction, the $(e+r-1) \times (e+r-1)$ summation matrix $\varPi$ naturally generates in the first array (Figure \ref{fig:ImplemenationForWinogradAlgorithm}), which would be multiplied with a transform matrix to deduce $e \times e$ outputs in the same channel of the small output sub-block. To complete the updata of each small sub-block with the size of $e \times e \times z$, each processor continuously loads the required inputs and weights (the red blocks in Figure \ref{fig:ImplemenationForWinogradAlgorithm}), and performs partial updates. In order to exploit the parallelism of the computation of $x \times y \times z$ outputs, each processor could use serval threads to execute the computation of all $e \times e$ tiles in a channel in parallel. For the update of each $x \times y \times z$ sub-block, every $e^2$ outputs rely on two $(e+r-1) \times (e+r-1)$ temporary arrays at a time. To maximize the data reuse of temporary arrays, we should use the most on-chip memory to store the $2xyz/e^2$ required temporary arrays. Hence, our design chooses $2 \frac{(e+r-1)^2}{e^2} xyz \approx S/N_{p}$.


In the dataflow above, an output image is divided into $(W_{out}H_{out}C_{out})/(xyz)$ sub-blocks. For each sub-block, we need to load $x'y'C_{in}$ inputs from an input image and $zr^2C_{in}$ weights from $z$ kernels. As $\mu = 1$ is only valid in Winograd algorithm, we have $x' \approx x$ and $y' \approx y$. The I/O volume for reading data can be estimated as follows
\begin{eqnarray}
Q_{WA~reading} \approx \frac{H_{out} W_{out} C_{out} }{x y z} \left( xy C_{in}  + z r^2 C_{in} \right) \nonumber \\
\geq H_{out} W_{out} C_{out} C_{in} \left( 2 \frac{  r}{\sqrt{x y z}} \right),
\end{eqnarray}
where the final equality holds if and only if $xy=r^2z$. Due to $R=r^2$ in Winograd algorithm, $xy=r^2z$ leads to $xy=Rz$, which is similar to the optimality condition for the dataflow of direct convolution. In addition, the I/O volume for writing outputs is $W_{out}H_{out}C_{out}$. As $2\frac{(e+r-1)^2}{e^2} xyz \approx S/N_{p}$, the total I/O volume is
\begin{displaymath}
Q_{WA} \approx \frac{2 H_{out} W_{out} C_{out} C_{in}r(e+r-1)}{e\sqrt{S/N_{p}}} + H_{out} W_{out} C_{out}. 
\end{displaymath}

As the proposed dataflow is similar to our design for the direct convolution, we just list two specific details in this design as follows:

\begin{itemize}
	\item The dataflow design of direct convolution mainly focuses on the output data reuse, while the dataflow design of Winograd algorithm is to exploit the data reuse of temporary arrays and combine input data reuse in the best way. In addition, each $(e+r-1)^2$ inputs are reused by weights of $z$ kernels, and each $r^2$ weights are reused by $e^2 $ outputs.
	\item The parallelism of the computation of $x \times y \times z$ outputs is fully considered. The update of every $e \times e$ tiles at an output channel is performed in parallel. To achieve a high parallelism and data reuse, the most on-chip memory is for loading the temporary arrays.
\end{itemize}

\section{Auto-Tuning for Implementation}

\subsection{Auto-Tuning Engine}
The dataflow design above just provides a coarse-grained strategies to minimize the off-chip memory access. In order to achieve an optimal implementation, fine-grained computational schedule and memory access schedule are still needed. In this section, we mainly consider the optimal implementation on accelerators, such as GPU. Similar optimization can be used for other hardware backends.

\begin{figure}[htbp]
	\centering
	\includegraphics[scale=0.6]{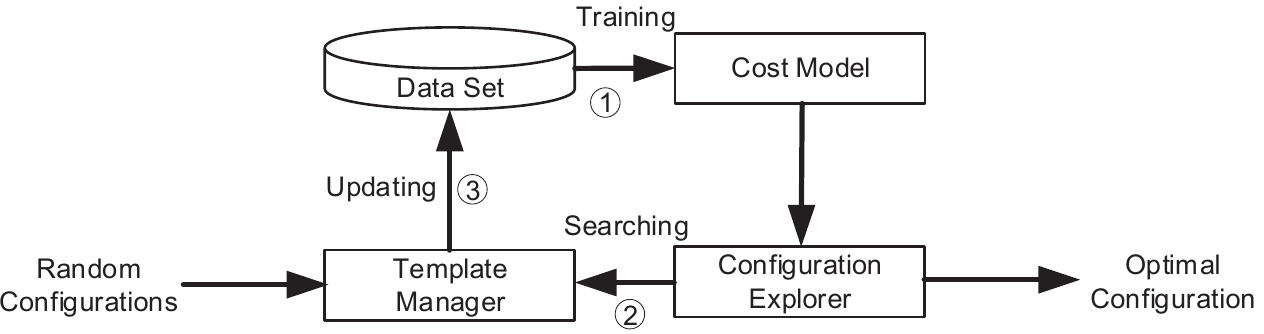}
	\caption{Auto-tuning Engine.}
	\label{fig:autotuning}
\end{figure}

For a given coarse-grained schedule, we define the configuration as a group of key performance parameters, including specific input shape and layout, number of threads in each thread block, tiling size, the shared memory size allocated to each thread block. Each configuration provides the description of an implementation way. All possible configurations constitute a configuration space whose size usually is over billions. In order to rapidly find the optimal choice in the huge space, we built an auto-tuning engine based on the learning-based cost modeling method. Figure \ref{fig:autotuning} shows the overview of our auto-tuning engine, which consists of three main components: a template manager that measures the execution time of any given configuration, and a cost model that predicts the cost of any given configuration, and a configuration explorer that searches promising new configurations.

\textbf{Template Manager:} In the low-level implementation, the proposed dataflow schedules are described as a template. Template manager is in charge of all schedule template, and generates various configurations for each template.
	
\textbf{Cost Model:} We use XGBoost method \cite{chen2016XGBoost} to train a gradient tree boosting model as the cost model to predict the runtime of any configuration. The model is trained using measurement data, which is consisted of a configuration and its execution time. During the auto-tuning process, the cost model would be updated periodically as the configuration explorer finds more configurations and updates the training dataset.
	
\textbf{Configuration Explorer:} During the configuration searching, the configuration explorer uses the trained cost model to predict the cost of any configuration, and searches the potential optimal configuration in the search space. Although the cost model could reduce the time to evaluate configurations, the searching process is still expensive due to a huge search space with over billions of size.

\subsection{Searching Based on Optimality Condition}

In order to improve the search efficiency, we construct a searching domain based on the optimality condition, which is helpful for significantly reducing the size of search space. Besides, we use a heuristic method to rapidly search promising configurations.

\textbf{Searching Domain:} Table \ref{tab:Configuration} presents the searching domain. According to the dataflow design, the tile is loaded into on-chip memory as a whole, which implies that $xyz \leq S_{b}$, $S_{b}$ is the shared memory size for each block. Furtermore, the optimality condition $xy=Rz$ leads to $z \leq \sqrt{S_{b}/R}$ and $xy \leq \sqrt{S_{b}R}$. In order to achieve a high level parallelism, at least two thread blocks are guaranteed to concurrently run on one streaming multiprocessor (SM), resulting in $S_{b} \leq S_{sm}/2$.

{
	\footnotesize
	\begin{table}[bhtp]\centering
		\vspace{-0.24em}
		\caption{Searching Domain.}\label{tab:Configuration}
		\vspace{-1.8em}
		\begin{center}
			\begin{tabular}{|c|c|}\hline
				Parameters & Definition and Constrains  \\\hline
				$H_{in}$, $W_{in}$, $C_{in}$ & Input shape \\\hline
				$H_{out}$, $W_{out}$, $C_{out}$ & Output shape \\\hline
				$H_{ker}$, $W_{ker}$ & Kernel shape \\\hline
				CHW, CWH, HWC &  Layout \\\hline
				$S_{sm}$ & Shared memory size of SM \\\hline
				\multirow{2}*{$S_{b}$} & Shared memory size for each block  \\
				& $S_b \leq S_{sm}/2$ \\\hline
				\multirow{2}*{$x$, $y$, $z$} & Tile size which are the factor of $H_{out}$, $W_{out}$, $C_{out}$, \\
				& $xyz \leq S_{b}$, $z \leq \sqrt{S_{b}/R}$ and $xy \leq \sqrt{S_{b}R}$  \\\hline
				$N_{xt}$, $N_{yt}$, $N_{zt}$ & Thread numbers which are  the factor of $x$, $y$, $z$  \\\hline
			\end{tabular}
		\end{center}
		\vspace{-1.0em}
	\end{table}
	
}

\textbf{Searching Process:} To find many promising configurations, the configuration explorer performs a searching process to select configurations from the searching domain. At the beginning of the searching process, $n_{s}$ random configurations are chosen as initial guesses. During each searching step, the configuration explorer randomly walks from each initial guesses to its nearby configuration in the searching domain. Each random walk tends to converge on a configuration that has lower predicted costs. Consequently, the $n_s$ parallel random walks generate $n_{s}$ promising configurations, which are saved as the initial guesses for the next searching step. Until all predicted costs of the $n_{s}$ selected configurations are lower than a threshold, they are outputted as a solution.

\subsection{Auto-tuning Process}

The proposed auto-tuning engine searches the optimal implementaion iteratively. Each iteration consists of three stages: (1) \textit{Model Training} that trains the cost model, (2) \textit{Configuration Searching} that applies the cost modal to select multiple promising configurations, (3) \textit{Dataset Updating} that measures the new configurations and updates the dataset. Until the measurement runtime of the selected configurations does not decrease for hundreds of iterations, the auto-tuning process would end. The parallel strategy corresponding the best selected configuration is the implementation of our near-optimal I/O dataflow.

\section{Evaluation}

In this section, we evaluate our proposed I/O optimal dataflow designs for the direct convolution and Winograd algorithm respectively. We first evaluate the optimal dataflow implementations derived from the proposed auto-tuning engine, and then compare the speeds of different automation searching methods, and finally demonstrate our implementation can achieve performance speedup in end-to-end cases. Our evaluation is mainly performed in the NVIDIA 1080Ti and V100 GPUs.

To evaluate our work from a broad scale, we use synthetic convolution cases with different $W_{ker}$ $H_{ker}$ and the stride $\mu$. On the one hand, in cuDNN library, the direct implementation of convolutions mainly has two approaches: direct convolution and image2col method \cite{jia2014learning}, where the direct convolution occasionally fails for some different input shapes, and the image2col method are usually better than the direct convolution. In order to present the superior of our implementations, we compare with the best one of two direct implementations in cuDNN. On the other hand, the indirect implementation of convolutions in cuDNN mainly is Winograd algorithm. The following evaluation compares the runtime of different convolution kernels of ours and cuDNN, where CUDA-9.0 and cuDNN-7.0.3 are used.

To evaluate the auto-tuning engine, we first compare the searching performance of our proposed searching method with different searching strategies in TVM, which represents the state-of-the-art technique for auto-tuning a convolution operation, and then compare our searched implementation with the optimal solution provided by TVM.

\subsection{Performance Comparison with cuDNN}

\begin{figure}[htbp]
	\centering
	\includegraphics[scale=0.24]{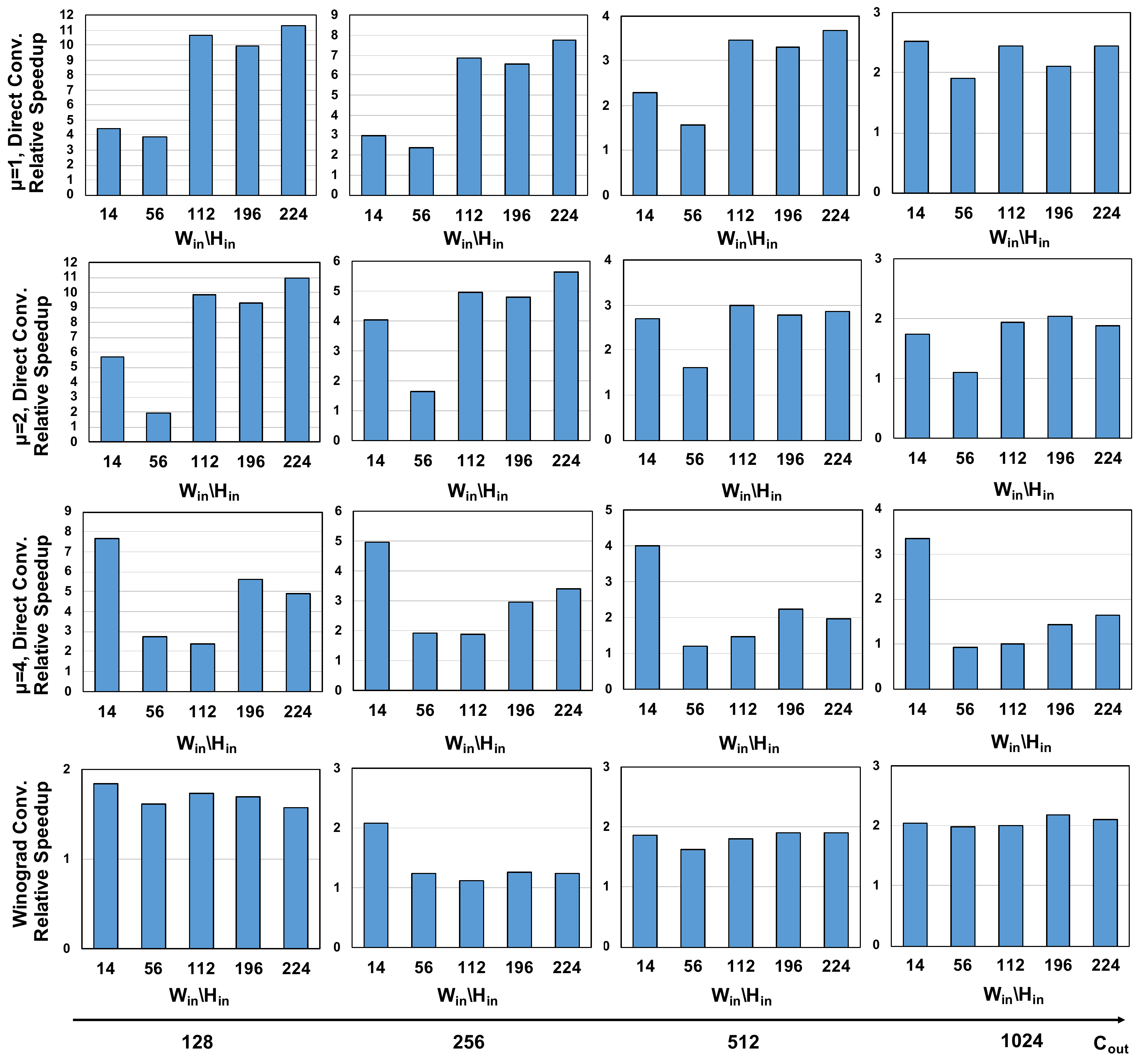}
	\caption{Performance Comparison of Dataflow Design over cuDNN for Direct Convolution and Winograd Algorithm on 1080Ti GPU. For all convolutions,  $H_{ker} \times W_{ker} = 3 \times 3$ and $C_{in} = 256$.}
	\label{fig:DCTest}
\end{figure}

Figure \ref{fig:DCTest} shows the performance comparison on the implementations of the direct convolution and Winograd algorithm respectively. We can find that our I/O optimal dataflow implementations can achieve $3.32 \times$ performance speedup on average. We have three important observations from the results. 

Firstly, the benefit from the dataflow is consistent as the $H_{in}$ and $W_{in}$ increase, and our methodology can have significant performance improvement. This mainly owes to the design of exploiting input and output data reuse. I/O dataflow design maximizes the data reuse of the $x' \times y'$ tile at a given channel. When $H_{in}$ and $W_{in}$ become larger, the more data reuse can be achieved.

Secondly, when $C_{out}$ is small, the dataflow contribution is always higher for the direct convolution. Conversely, when $C_{out}$ is large, the benefit from the dataflow is always higher for Winograd algorithm. 

Third, on the whole, the dataflow benefits decrease as the stride $\mu$ increase. This is because the motivation of I/O dataflow design is to minimize the off-chip memory access. When the stride $\mu$ is larger, more off-chip memory accesses gradually become independent with each other.

Furthermore, Figure \ref{fig:BatchTest} shows the batched convolution test. It is clear that, compared with scaling the batch size of cuDNN, our I/O optimal dataflow still achieves $1.51 \times$ performance speedup on average. On the one hand, For a given batch-size, when $H_{in}$ and $W_{in}$ increases, the performance improvement from our dataflow design gradually becomes apparent. On the other hand, when $H_{in}$ and $W_{in}$ are small, the dataflow contribution is small. However, when $H_{in}$ and $W_{in}$ become larger, the convolution needs more I/O operations, and the benefit from the dataflow becomes greater. When $H_{in}$ and $W_{in}$ are $112$, the speedup becomes larger with the batch size increasing.

\begin{figure}[htbp]
	\centering
	\includegraphics[scale=0.43]{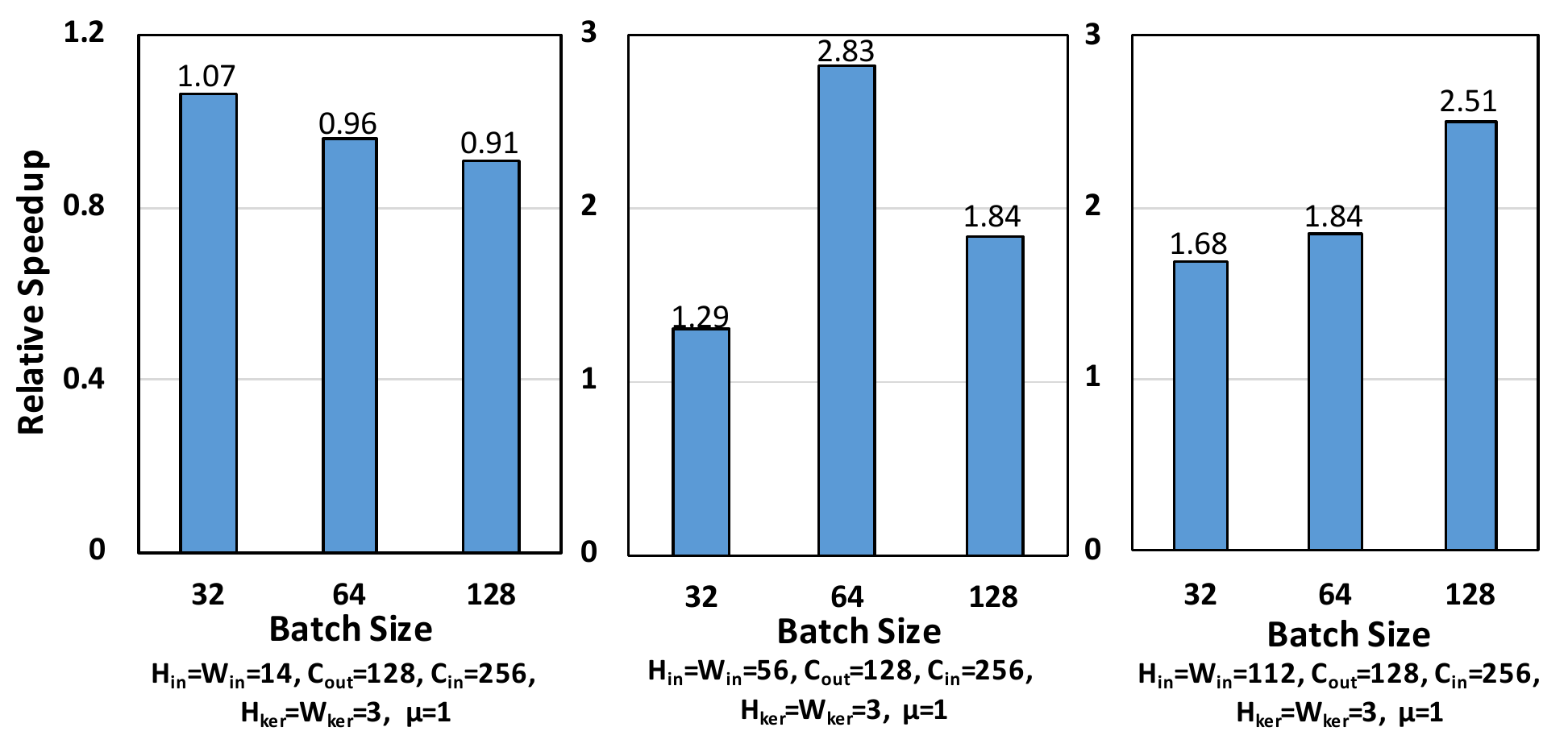}
	\caption{Performance Comparison of Dataflow Design over cuDNN for Batched Direct Convolution Test on 1080Ti GPU.}
	\label{fig:BatchTest}
\end{figure}

\subsection{Performance Comparison with TVM}

\begin{figure}[htbp]
	\centering
	\includegraphics[scale=0.36]{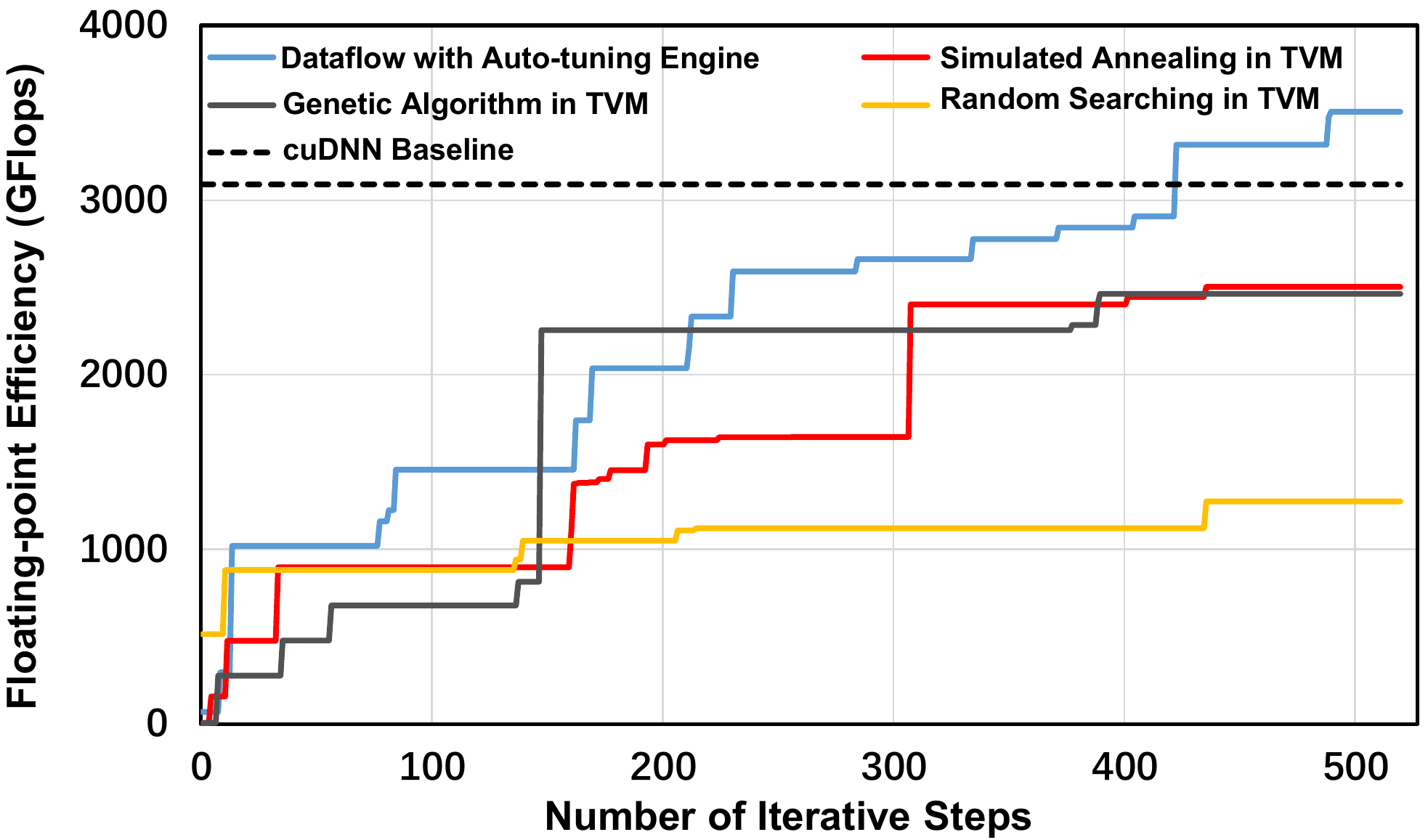}
	\caption{Comparison of Different Automation Methods.}
	\label{fig:AutoTuningCurve}
\end{figure}

{

	\footnotesize
	\begin{table*}[bhtp]\centering
		\vspace{-1.0em}
		\caption{Comparison of TVM with Auto-tuning Engine (ATE).}\label{tab: Comparison of TVM with ATE}
		\vspace{-1.5em}
		\scriptsize
		\begin{tabularx}{17.9cm}{p{1.1cm}<{\centering}|p{0.2cm}<{\centering}p{0.8cm}<{\centering}p{0.5cm}<{\centering}p{1.1cm}<{\centering}p{0.4cm}<{\centering}p{0.7cm}<{\centering}|p{1.2cm}<{\centering}p{1.2cm}<{\centering}p{0.9cm}<{\centering}|p{0.3cm}<{\centering}p{0.3cm}<{\centering}p{0.9cm}<{\centering}|p{0.8cm}<{\centering}p{0.8cm}<{\centering}p{1cm}<{\centering}}\toprule
			\multirow{2}*{Convolution} &\multicolumn{6}{c}{Parameter}&\multicolumn{3}{c}{Size of Search Space}  & \multicolumn{3}{c}{Iterations} & \multicolumn{3}{c}{Performance of Solution (GFlops)} \\\cmidrule{2-16}
			&$C_{in}$ & $~H_{in}/W_{in}~$ & $~C_{out}~$ & $~H_{ker}/W_{ker}~$ & stride & padding &TVM & ATE &ATE/TVM &TVM & ATE & TVM/ATE &TVM & ATE &ATE/TVM \\\midrule
			conv1 &3 &227 &96 &11 &4 &0 &$9.29 \times 10^6$ &$4.81 \times 10^6$ &51.78\% &142 &197 &0.72 &2927.30 &5377.06 &1.84 \\
			conv2 &96 &27 &256 &5 &1 &2 &$2.25 \times 10^8$ &$4.76 \times 10^7$ &21.16\% &762 &449 &1.53 &5909.73 &6426.83 &1.09 \\
			conv3 &256 &13 &384 &3 &1 &1 &$1.87 \times 10^7$ &$4.51 \times 10^6$ &24.12\% &877 &389 &2.25 &2107.68 &2555.93 &1.21 \\
			conv4 &384 &13 &256 &3 &1 &1 &$1.54 \times 10^7$ &$5.23 \times 10^6$ &33.96\% &784 &407 &1.93 &2040.57 &2040.92 &1.00 \\
			conv3\_wino &256 &13 &384 &3 &1 &1 &$2.59 \times 10^5$ &$1.36 \times 10^5$ &52.51\% &352 &202 &1.74 &6700.77 &6726.17 &1.01 \\
			conv4\_wino &384 &13 &256 &3 &1 &1 &$1.58 \times 10^5$ &$8.06 \times 10^4$ &51.01\% &587 &286 &2.05 &7121.57 &7118.23 &1.00 \\
			\bottomrule
		\end{tabularx}
	\end{table*}
}

Table \ref{tab: Comparison of TVM with ATE} presents the detail information about configuration space, the number of iterations and the best solution's runtime of the auto-tuning engine and TVM during searching the optimal implementations of different convolution layers in AlexNet on V100 GPU. We have three important observations from the experiment results. Firstly, the constraints for the templates and the proposed searching domain can successfully reduce the size of configuration space to about $20\%-50\%$ for the direct convolution and $50\%$ for Winograd algorithm. The compression ratio for Winograd algorithm is not small, because the size of original configuration space is small (see the space size in TVM) and the flexibility for implementation design is limited essentially. Secondly, the proposed auto-tuning engine finds the final solution faster than TVM, thanks to the proposed searching domain. Thirdly, the final configuration found by the auto-tuning engine usually has a shorter runtime than the best solution in TVM. The three facts above demonstrate that the auto-tuning engine has the strong scaling efficiency for searching optimal configuration. 

Figure \ref{fig:AutoTuningCurve} shows the comparison of different automation methods for searching an optimal direct convolution implementation of the conv1 in Table \ref{tab: Comparison of TVM with ATE} on V100 GPU. The ML-based model in TVM starts with no training data and uses the collected data to improve itself. The X-axis is the number of iterative steps and the Y-axis is the floating-point arithmetic efficiency in GFlops. From Figure \ref{fig:AutoTuningCurve}, we observe a similar trend for all automation methods. During the iterations, each automation method gradually finds the better configuration with higher floating-point arithmetic efficiency. It should be noted that the proposed auto-tuning engine is able to find better configurations much faster than the others. This mainly owes to two reasons. On the one hand, the I/O optimality condition is used to prune configuration search space, which leads to the proposed searching domain. On the other hand, the parallel searching method effectively improves the searching process in the searching domain.

\subsection{Performance Comparison on CNN Models}

The modern CNN models introduce many layer structures, such as convolution layer. More specifically, the convolution layer is important and popular in many state-of-the-art CNN models such as ResNet \cite{szegedy2016inception}, VggNet, SqueezeNet \cite{iandola2016squeezenet} and so on. In the following, we demonstrate that our proposed auto-tuning engine can help for accelerating CNN inference. 


Figure \ref{fig:end-to-end} shows the performance comparison of the dataflow design and cuDNN on different CNN models. For SqueezeNet, Vgg-19, ResNet-18, ResNet-34 and Inception-v3, our optimal implementation can achieve $2.67 \times$, $1.09 \times$, $1.02 \times$, $1.09 \times$ and $1.23 \times$ performance speedup respectively compared with using cuDNN. The performance benefits come from two aspects. The different kinds of convolutions take up the main part of CNN models. Besides, for each convolution layer, the proposed auto-tuning engine could find a better implementation than cuDNN.

\begin{figure}[htbp]
	\centering
	\includegraphics[scale=0.58]{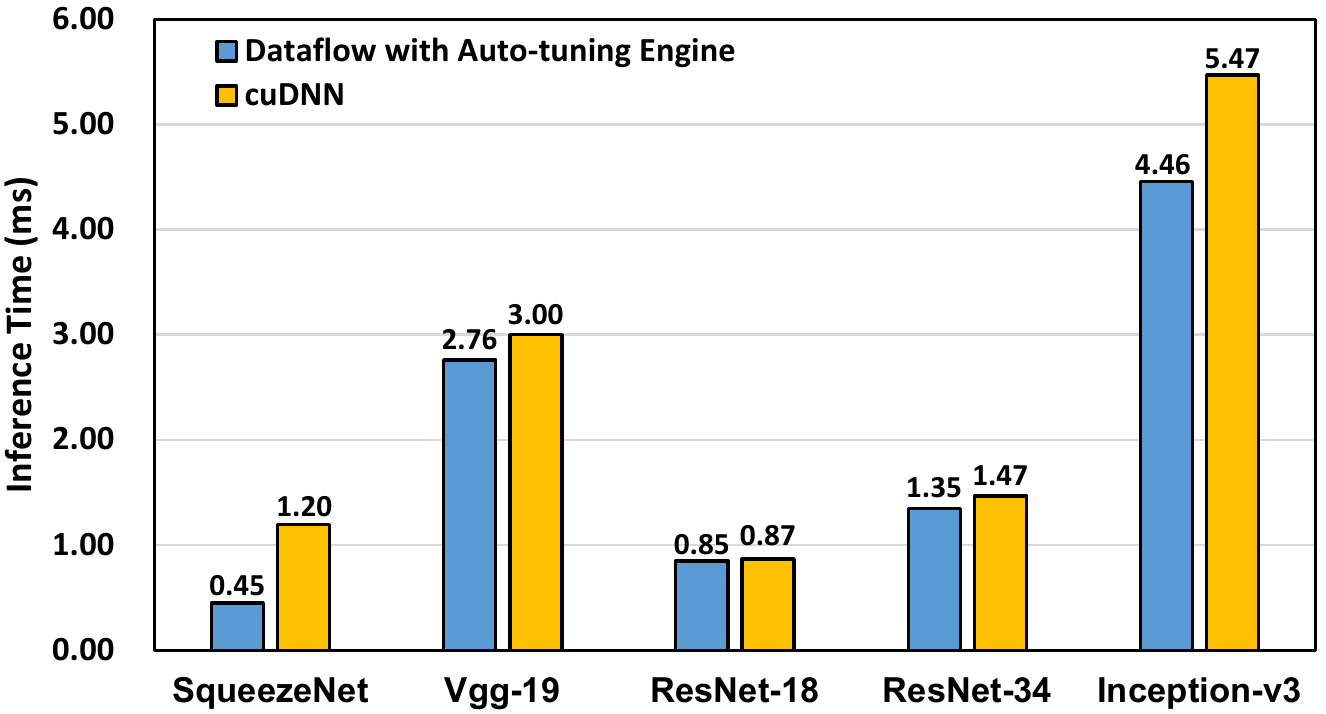}
	\caption{Performance Comparison of Dataflow Design over cuDNN on different CNN Models on V100 GPU.}
	\label{fig:end-to-end}
\end{figure}

\subsection{Sensitivity for GPU Architecture}

\begin{figure}[htbp]
	\centering
	\includegraphics[scale=0.46]{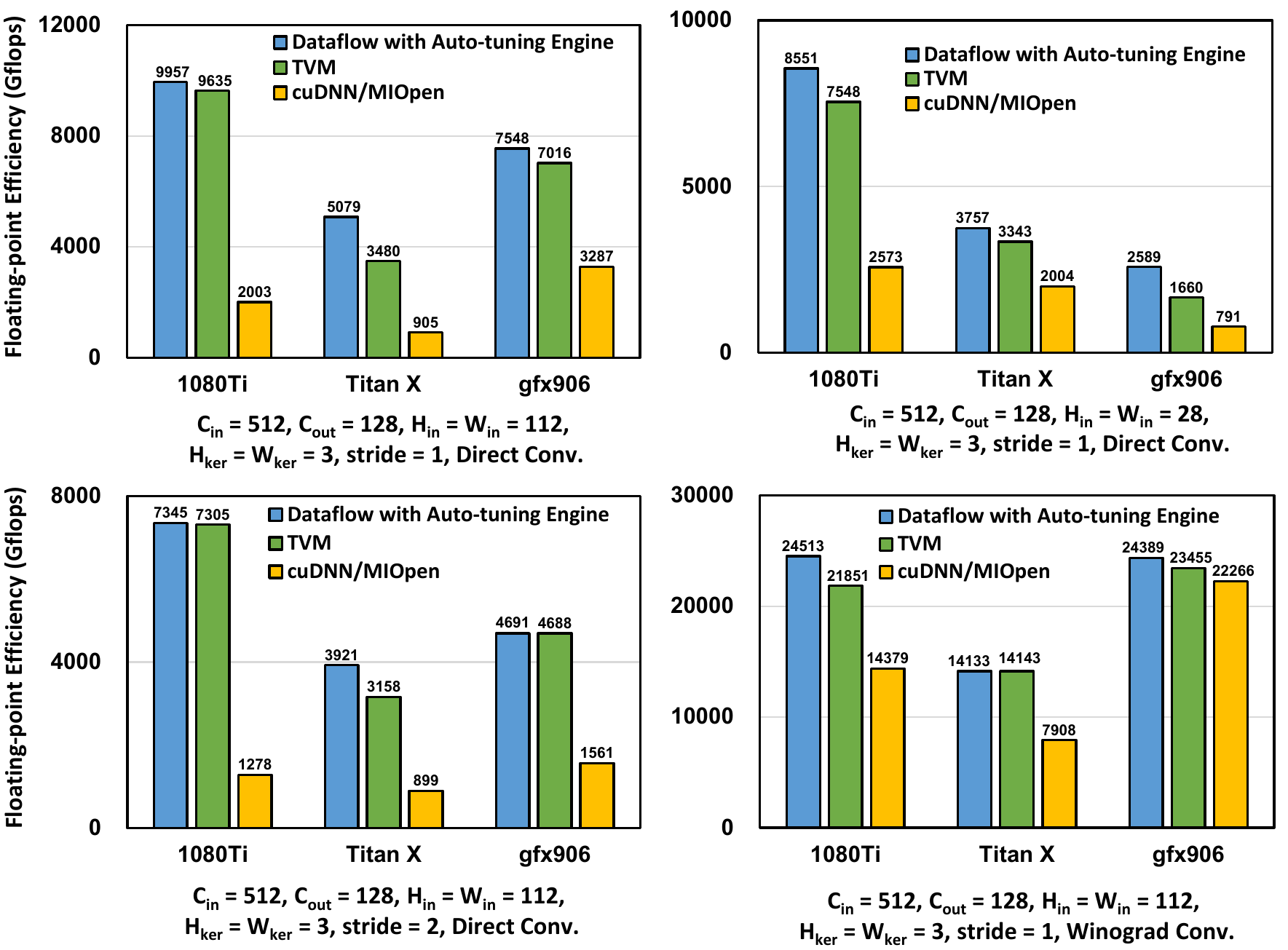}
	\caption{Sensitivity on different GPU architectures.}
	\label{fig:MultiPlatform}
\end{figure}

To demonstrate the scalability on GPU architecture, we evaluate the proposed dataflow with auto-tuning engine on Pascal and Maxwell architectures. We use one kind of Pascal architectures: 1080Ti, and one kind of Maxwell architecture: GTX Titan X. Figure \ref{fig:MultiPlatform} shows the evaluation results on the above two architectures. The proposed dataflow is much faster than cuDNN. Compared with the solution of TVM, for the direct convolution, the improvement of our implementation on these architectures can achieve about $1.05 \times$ and $1.27 \times$ respectively. For Winograd algorithm, the speedups of our dataflow are $1.12 \times$ and $1.01 \times$ respectively on these architectures.

In addition, we compare the dataflow design with MIopen library on AMD GFX906 platform (Pre-Wukong GPU), and use ROCm-2.9 and MIopen-2.1 in this evaluation. On average, the performance improvement is up to $2.86 \times$ and $1.10 \times$ for direct convolution and Winograd algorithm respectively. Besides, compared with the solution of TVM, our optimal implementation achieves $1.21 \times$ speedup for the direct convolution and $1.03 \times$ speedup for Winograd algorithm. We find that our optimal implementation is well ported to different architectures and achieve a consistent performance speedup.


\section{Related Work}
The red-blue pebble game is widely used in theory analysis of communication lower bound to guide optimal communication strategy. After Hong \& Kung established the I/O complexity theory \cite{jia1981complexity}, Savage developed the notion of S-span to derive Hong-Kung style lower bounds \cite{savage1995extending}. Kwasniewski et al. provided a new proof of I/O complexity of matrix-matrix multiplication and designed a parallel algorithm to reach its lower bound \cite{kwasniewski2019red}. Although the red-blue pebble game model has been proposed for many years \cite{aggarwal1988input,savage1998models,demmel2012communication,Ballard2011minimizing,ballard2013graph,solomonik2013minimizing}, it is still difficult to use this model to establish I/O lower bounds of composite algorithms which involve several different kinds of computational patterns \cite{elango2014characterizing}.
To get around the essential difficulties, the lower bound of composite algorithms was considered by modifying the red-blue pebble game model into a red-blue-white pebble game model \cite{elango2014characterizing}, which uses some restrictions on models, such as the limitation of disallowing re-computation of values on the DAG \cite{elango2014characterizing}. However, such restrictions seem inappropriate for the lower bound analysis of some convolution algorithms. For example, Winograd algorithm allows re-computation of values to decrease the number of I/O operations. In order to solve the difficulties, this work at first establishes a general I/O lower bound theory for any composite algorithm based on the red-blue pebble game model without introducing the limitation of disallowing re-computation of values on the DAG.
	
For convolutions in DNN, Demmel et al. estimated the minimum memory access of direct convolution by solving an intricate optimization problem \cite{demmel2018communication}.
Furthermore, Chen et al. transformed the direct convolution into Matrix-matrix multiplication, and successfully deduced the lower bound of the off-chip communication of direct convolution in CNN accelerators \cite{chen2020hpca}. However, our work is the first time to perform a systematic analysis of diverse convolution algorithms in deep learning by developing a general I/O lower bound theory for any composite algorithm. It is worth mentioning that the I/O lower bound in Equation (\ref{equation: communication lower bound of DC}) is equivalent to the I/O lower bounds of direct convolution in \cite{demmel2018communication,chen2020hpca}, while our proposed result on direct convolution is the tighter lower bound with a more precise coefficient. Besides, the previous works \cite{demmel2018communication,chen2020hpca} mainly focus on the direct convolution, and seem not easy to adapt to Winograd algorithm. However, to the best our knowledge, this work at first establishes the I/O lower bound of Winograd algorithm.

To fully exploit the research efforts from convolution algorithm and micro-architecture optimizations, many software libraries, such as cuDNN, are launched to pack these optimizations together in order to reduce programming difficulty. However, due to the increasing demand on performance, directly using the software libraries sometimes is not satisfactory. In recent years, the convolution optimization is widely concerned. Some excellent implementations are proposed for different convolution algorithms \cite{chen2016eyeriss,shah2018runtime,peemen2013memory,shi2015locality,jo2018energy}. However, most of the studies mainly focus on the optimization from experience\cite{zhang2019ispa}. In this work, we try to propose the I/O optimal dataflow based on the lower bound theoretical analysis. By comparing the I/O volume of the dataflow with the lower bound, we find the optimality condition for I/O optimal design. On the other hand, in the convolution optimization, the combinatorial choices of memory access, threading pattern, and novel hardware primitives creates a huge configuration space. A common way is to adopt a predefined cost model to guide the search, but building an accurate cost model is difficult due to the increasing complexity of modern hardware. To addresses these challenges, some searching strategies based on the learning-based cost models are proposed, in which TVM represents the state-of-the-art auto-tuning technique. However, it still needs a large search cost due to the huge search space. In this work, this work firstly considers to use the deduced optimality condition to fully reduce the size of search space, and proposes an effective parallel searching method to find the optimal implementation, which leads to an effective auto-tuning engine. Compared with TVM, it could faster find a better final solution.

\section{Conclusion}
In this paper, we have tackled the challenge of building I/O lower bound theory and designing I/O optimal dataflow implementations for convolutions. By fine-grain viewing the recent lower bound theory developed under the red-blue pebble game model, we fully consider the influence of sub-computations to each other, and propose a general I/O lower bound theory for composite algorithms. Based on the proposed theory, we establish the communication lower bound results for the typical representatives of direct and indirect convolution methods, which are the direct convolution and Winograd algorithm. Furthermore, for each approach, we design the I/O optimal dataflow strategy based on the lower bound analysis. By developing an auto-tuning engine for searching the optimal configuration, we push the envelope of performance of our dataflow designs further.

\begin{acks}
The authors would like to thank all anonymous referees for their valuable comments and helpful suggestions. The work is supported by National Key Research and Development Program of China under Grant No. (2018AAA0103302, 2016YFC1401706, 2016YFB0200800), National Natural Science Foundation of China under Grant No. (62032023, 61802369) and Huawei Technologies Co., Ltd.. The authors also thank Dr. Long Wang and the group of Huawei Technologies Co., Ltd. for their help to this research.
\end{acks}